\documentclass[10pt,letterpaper]{article}
\usepackage[T1]{fontenc}
\usepackage[utf8]{inputenc}
\usepackage{nsfm_preprint}
\usepackage{amsmath,amssymb,amsthm,mathtools,bm}
\usepackage{booktabs,array,longtable,enumitem,microtype}
\usepackage[round]{natbib}
\usepackage[hyphens]{url}
\usepackage{tikz}
\usetikzlibrary{arrows.meta}
\usepackage[hidelinks]{hyperref}
\newcounter{algorithmctr}
\hypersetup{pdftitle={Minimax-Optimal Predictive Flow Matching under Distribution Drift},pdfauthor={Aadithya Srikanth and Vaneet Aggarwal},pdfsubject={Predictive flow matching under distribution drift},pdfkeywords={flow matching, dynamic regret, nonstationary learning, Wasserstein drift, minimax lower bounds}}
\newcommand{\R}{\mathbb R}
\newcommand{\E}{\mathbb E}
\newcommand{\Prob}{\mathbb P}
\newcommand{\V}{\mathcal V}
\newcommand{\Reg}{\mathcal R_T^{\mathrm{vel}}}
\newcommand{\App}{\mathcal A_T}
\newcommand{\eopt}{\epsilon_{\mathrm{opt}}}
\newcommand{\ind}{\mathbf 1}
\newcommand{\norm}[1]{\lVert#1\rVert}
\newcommand{\abs}[1]{\lvert#1\rvert}
\newcommand{\osc}{\operatorname{osc}}
\newcommand{\Targ}{\mathsf T_T}
\newcommand{\Fast}{\mathsf F_T}
\newcommand{\Lsm}{\Lambda_{\mathrm{sm}}}
\newcommand{\esm}{\epsilon_{\mathrm{sm}}}
\newcommand{\Hel}{\mathsf H}
\DeclareMathOperator*{\argmin}{arg\,min}
\newtheorem{theorem}{Theorem}
\newtheorem{lemma}{Lemma}
\newtheorem{proposition}{Proposition}
\newtheorem{corollary}{Corollary}
\theoremstyle{definition}
\newtheorem{assumption}{Assumption}

\newtheorem{remark}{Remark}
\newcommand{\supppage}{\clearpage}
\allowdisplaybreaks[1]

\title{Minimax-Optimal Predictive Flow Matching under Distribution Drift}
\author{Aadithya Srikanth\thanks{Correspondence to: \texttt{srikan18@purdue.edu}.}\qquad Vaneet Aggarwal\\[4pt]
\normalfont \mbox{Purdue University, West Lafayette}\\
\normalfont\texttt{\{srikan18, vaneet\}@purdue.edu}}
\date{}
\renewcommand{\shorttitle}{Predictive Flow Matching under Distribution Drift}

\begin{document}
\maketitle
\begin{abstract}
Flow matching learns a generative velocity field by regression, but under distribution drift a field must be deployed before the current data arrive. We study this predictive problem for a fixed, possibly nonconvex, neural class, with dynamic regret against the best represented field of each round. We propose DyProd-FM, which refits flow matching on dyadic historical prefixes and selects among the fits with a sleeping Prod master. Without knowing any drift budget, it attains regret $\widetilde O(\min\{N_\omega+1,1+\sqrt{T\Omega}\})$ in the relative velocity risk variation $\Omega$ and $\widetilde O(\min\{N_\mu+1,1+T^{1/3}V_\mu^{2/3}\})$ in the target-law variation $V_\mu=\sum_tW_2(\mu_{t+1},\mu_t)$, up to approximation and fitting costs. Matching lower bounds for fixed classes show that both rates are optimal in $T$ and in the budgets up to logarithmic factors. Prior optimal dynamic-regret guarantees rely on a convex decision set with strongly convex or exp-concave losses, whereas our class may be nonconvex and no curvature in its parameters is assumed. The target-law rate rests on a second-order link between the two variations, which the base noise of flow matching provides.
\end{abstract}

\section{Introduction}\label{sec:intro}
Flow matching constructs a generative model by regressing a velocity field along a path from a base distribution to a target distribution \citep{lipman2023,liu2023,albergo2025}. Its statistical analysis usually fixes the target while observations accumulate. In a changing environment, this separation between training and deployment is insufficient: a field learned from a large historical sample may accurately solve an obsolete regression problem. Conversely, discarding history after every observation forfeits the statistical benefit of repeated sampling. A predictive theory must quantify both effects under the information available when a field is deployed.

We study nonstationary stochastic flow matching with independent, fresh batches. A learner selects one represented field using past batches, then receives the current batch and scores every active candidate. Its benchmark is the best represented field for the current population, which may change every round. Learning rounds, along which the environment changes, are distinct from flow time, along which a selected field generates a sample.

How should drift be measured? The \emph{relative velocity risk variation} $\Omega$ measures how drift changes the comparison between represented fields. It is class-relative and ignores changes the class cannot see. The \emph{target-law variation} $V_\mu=\sum_tW_2(\mu_{t+1},\mu_t)$ measures drift of the generative target itself. It is class-free and estimable from samples. We determine the minimax rate under both measures, up to logarithmic factors, for a fixed class and fixed problem scales.

Three difficulties arise. Historical batches are independent but heterogeneous, so one common comparator need not fit all past velocities. Historical fitting attains $O(1/n)$ with $n$ examples, but standard online selection adds order $\sqrt T$ regret and erases this gain. First-order conversion of target drift into risk drift discards squared-loss curvature and yields only $\sqrt{TV_\mu}$. We must resolve all three for a nonconvex class without parameter-curvature assumptions. Section~\ref{sec:analysis} resolves them in turn through historical estimation (Section~\ref{sec:hist}), fast selection (Section~\ref{sec:aggregation}), and drift transfer followed by horizon partitions (Sections~\ref{sec:drift}--\ref{sec:architecture}).

\paragraph{Contributions.}
\begin{enumerate}
\item \textbf{DyProd-FM and optimal rates in both variations.} DyProd-FM refits flow matching on deterministic, geometrically growing historical prefixes and properly selects one fit per round with a sleeping Prod master at a single fast rate \citep{gaillard2014}. It receives no budget. Its dynamic velocity regret is $\widetilde O(\min\{N_\omega+1,1+\sqrt{T\Omega}\})$ in the relative velocity risk variation (Theorem~\ref{thm:main}) and $\widetilde O(\min\{N_\mu+1,1+T^{1/3}V_\mu^{2/3}\})$ in the target-law variation (Theorem~\ref{thm:target}), up to approximation and fitting costs. We complement these with matching lower bounds for fixed classes at fixed problem scales (Appendix~\ref{app:lower}), showing that both rates are optimal in $T$ and in the budgets up to logarithmic factors. The target-law exponents are those of total-variation forecasting \citep{baby2019}, here for a nonconvex neural class and a proper learner.
\item \textbf{A second-order link between the two variations.} A target move is exactly a translation of the base noise plus a translation of the label at the same input (Equation~\eqref{eq:baseshift}). Base smoothing controls how this change affects comparisons between fields (Lemma~\ref{lem:smoothing}). For a field optimal under an earlier target, the drift contribution to excess risk is bounded by a constant times the squared Wasserstein distance to the new target, with additional approximation and smoothing remainders. The same estimate also bounds relative-risk variation $\Omega$ in terms of target-law variation $V_\mu$ (Proposition~\ref{prop:link}). The smoothing modulus is sharp for Gaussian bases, and the moment condition it requires excludes uniform flow time and admits the truncated uniform, Beta, and logit-normal laws used in practice (Proposition~\ref{prop:necessity}).
\item \textbf{Analysis beyond fixed-target flow matching.} The fast historical rate uses the Bernstein-type variance control of squared-loss flow matching \citep{sahoo2026}, in the approximate form implied by the regression identity, and our lower bounds build on the rare-mixture construction of the same work. What is new is the nonstationary part. (i) Histories are independent but not identically distributed, and their common comparator carries a drift-induced residual even when every round is realizable (Proposition~\ref{thm:historical}). (ii) The same squared-loss structure bounds the second moment of the master's score differences by velocity risks, so selecting among the fits costs a polylogarithmic factor in $T$ per lifetime, rather than the square root of the lifetime's length. (iii) Relative-risk oscillation gives first-order historical transfer in $\Omega$, while base smoothing and approximate quadratic growth give second-order transfer in $V_\mu$. (iv) The two-point constructions become block constructions for dynamic regret that keep every class constant fixed.
\end{enumerate}

\section{Related Work}\label{sec:related}
\paragraph{Flow matching.}
Flow matching \citep{lipman2023,lipman2024guide}, rectified flow \citep{liu2023}, and stochastic interpolants \citep{albergo2023,albergo2025} learn transport velocities by regression. Fixed-target theory relates velocity accuracy to sampling error \citep{benton2024} and bounds statistical errors and sample complexity \citep{zhou2025,gaur2025}. \citet{sahoo2026} obtain order-optimal fixed-target rates from Bernstein-type variance control, with a rare-mixture lower bound. Nonuniform flow-time laws are standard in large-scale training \citep{esser2024}. We ask instead how much excess current velocity error accumulates when deployment precedes each fresh batch and the target drifts.

\paragraph{Dynamic regret.}
Tracking changing comparators is classical \citep{herbster1998,zinkevich2003,besbes2015}. Our analysis builds on strongly adaptive interval guarantees \citep{daniely2015,jun2017,zhang2018,wei2016} and on Prod with sleeping experts \citep{gaillard2014}. For $n$ observations with total variation $C_n$ under squared loss, \citet{baby2019} establish the minimax rate $n^{1/3}C_n^{2/3}$, and \citet{baby2021,baby2022} attain such rates with proper learners for exp-concave and strongly convex losses over convex decision sets. Our target-law rate has the same exponents for a nonconvex neural class, where the curvature comes from the flow-matching regression identity and base smoothing rather than from the loss geometry in parameter space. We also use fast-rate arguments \citep{bartlett2005,srebro2010,blm2013} and two-point lower bounds \citep{tsybakov2009,sahoo2026}. Appendix~\ref{app:related} gives the extended discussion.

\section{Problem Setting}\label{sec:setting}
\subsection{Sampling and Velocity Risk}
Let $d$ be the data dimension, $T\ge1$ the number of learning rounds, and $m\ge1$ the batch size. Learning rounds are indexed by $t$, and flow time is $r\in[0,1]$. Throughout, $\norm{\cdot}$ is the Euclidean norm, and matrix operator and Frobenius norms are denoted by $\norm{\cdot}_{\mathrm{op}}$ and $\norm{\cdot}_F$. A fixed base law is denoted by $\nu$ and the round-$t$ target by $\mu_t$, both on $\R^d$. An exogenous standard Borel environment $E$ fixes the target sequence through jointly measurable kernels. Conditional on $E$, all endpoint-time triples are independent, and the $m$ triples within round $t$ have common law
\begin{equation}\label{eq:protocol}
 (X_0,X_1,r)\sim\nu\otimes\mu_t\otimes\lambda,
\end{equation}
where $\lambda$ is a fixed flow-time law on $[0,1]$. The relative-risk results use $\lambda=\mathrm{Unif}[0,1]$ by default, and the target-law results use a law satisfying Assumption~\ref{ass:smooth}. The environment does not respond to the learner. All envelopes and budget caps below are deterministic and uniform over environments. We write $\E_t$ for expectation under one fresh triple at round $t$, conditional on $E$. The outer $\E$ also averages over data, independent algorithmic seeds, and, if random, the environment.

Let $I_r(x_0,x_1)$ be the fixed interpolation, and let $Z_r(x_0,x_1)$ denote its flow-time derivative. The regression input and observed label are, respectively,
\begin{equation}\label{eq:example}
 U=(I_r(X_0,X_1),r),\qquad Z=Z_r(X_0,X_1).
\end{equation}
For a fixed represented class $\V$, define the ideal regression field $b_t$, population loss $L_t$, and absolute velocity risk $Q_t$ by
\begin{align}
 b_t(U)&=\E_t[Z\mid U],\label{eq:ideal}\\
 L_t(v)&=\E_t\norm{v(U)-Z}^2,\nonumber\\
 Q_t(v)&=\E_t\norm{v(U)-b_t(U)}^2.\label{eq:risk}
\end{align}
All conditional means are formed after fixing $E$. The sampling law of $U$ may change with $t$, so these are population risks under the current interpolation law. We write $\norm f_t^2=\E_t\norm{f(U)}^2$ for the field norm under the round-$t$ input law.

The best represented approximation error at round $t$ is $a_t$, and $F_t(v)$ denotes excess velocity risk:
\begin{equation}\label{eq:excess}
 a_t=\min_{v\in\V}Q_t(v),\qquad F_t(v)=Q_t(v)-a_t.
\end{equation}
The assumptions below ensure that the minimum exists. Conditional-expectation orthogonality gives
\begin{equation}\label{eq:regression}
 L_t(v)=Q_t(v)+\E_t\norm{Z-b_t(U)}^2.
\end{equation}
Consequently $F_t(v)=L_t(v)-\min_{u\in\V}L_t(u)$. If $v_t\in\V$ is deployed at round $t$, define cumulative dynamic velocity regret and cumulative approximation error by
\begin{equation}\label{eq:regret}
 \Reg=\sum_{t=1}^T F_t(v_t),\qquad \App=\sum_{t=1}^T a_t.
\end{equation}
Thus $\sum_tQ_t(v_t)=\Reg+\App$ exactly, and since $L_t-F_t$ does not depend on the field, minimizing $\sum_tF_t(v_t)$ is minimizing $\sum_tL_t(v_t)$. Regret measures competition with the current class optimum, and approximation measures the class's remaining distance to the ideal velocity.

\subsection{Two Measures of Nonstationarity}\label{sec:budgets}
\paragraph{Relative velocity risk variation.}
We measure class-relative nonstationarity by the half-oscillation of a one-round population-risk increment. For $1\le t<T$, define
\begin{align}
 \omega_t=\frac12\sup_{u,v\in\V}\big|&[L_{t+1}(v)-L_{t+1}(u)]\nonumber\\[-2pt]
 &-[L_t(v)-L_t(u)]\big|.\label{eq:variation}
\end{align}
Replacing every $L_t$ by $Q_t$ leaves this quantity unchanged. Let $\Omega_T$ be total relative-risk variation and $N_T^\omega$ its number of nonzero increments. We use almost-sure caps $\Omega$ and $N_\omega$:
\begin{equation}\label{eq:caps}
 \Omega_T=\sum_{t<T}\omega_t\le\Omega,\quad
 N_T^\omega=\sum_{t<T}\ind\{\omega_t>0\}\le N_\omega.
\end{equation}
These caps appear only in the analysis. They are not supplied to the learner. Common additive shifts of all fields' risks contribute zero variation.

The budget compares complete current risks under their respective input laws while removing the field-independent noise term in Equation~\eqref{eq:regression}, and it is class dependent. Its half-oscillation form gives $\sup_{v\in\V}|F_{t+1}(v)-F_t(v)|\le2\omega_t$.

\paragraph{Target-law variation.}
Let $W_2$ be the quadratic Wasserstein distance on $\R^d$. Define the target-law variation, its change count, and their caps by
\begin{align}
 V_T^\mu&=\sum_{t<T}W_2(\mu_{t+1},\mu_t)\le V_\mu,\nonumber\\
 N_T^\mu&=\sum_{t<T}\ind\{\mu_{t+1}\ne\mu_t\}\le N_\mu.\label{eq:targetvariation}
\end{align}
This budget does not depend on the class. With a fixed base, flow-time law, and independent pairing, $\mu_{t+1}=\mu_t$ forces $L_{t+1}=L_t$, so $N_T^\omega\le N_T^\mu$. Under Assumption~\ref{ass:smooth}, Proposition~\ref{prop:link} bounds $\Omega_T$ by a multiple of $V_T^\mu$ up to small remainders, but no class-independent equivalence holds in either direction. 

\subsection{Assumptions}\label{sec:assumptions}
Assumptions~\ref{ass:tails}--\ref{ass:fit} follow the standard setting of fixed-target flow-matching sample-complexity analysis \citep{sahoo2026}: sub-Gaussian endpoints, a label envelope for the path, a bounded feed-forward network, and access to an empirical minimizer, which Assumption~\ref{ass:fit} relaxes to an approximate one.

\begin{assumption}[Endpoint envelope]\label{ass:tails}
A known $\kappa>0$ bounds both endpoints, uniformly in $t$ and conditionally on $E$:
\begin{equation}\label{eq:tails}
 \E_t\exp(\norm{X_j}^2/\kappa^2)\le2,\qquad j\in\{0,1\}.
\end{equation}
\end{assumption}
\begin{remark}
This uncentered tail condition controls moving means as well as fluctuations and permits unbounded support. It implies $\E_t\norm{X_j}^2\le\kappa^2\log2$, so $W_2$ distances between targets are finite, and it lets a deterministic endpoint gate lose only an exponentially small part of every risk (Proposition~\ref{prop:tail}). For isotropic laws $\kappa^2$ grows at least linearly in $d$, and $\kappa$ sets the gate radius in Equation~\eqref{eq:gatechoice}, so the constants are not dimension-free.
\end{remark}

\begin{assumption}[Regular compatible path]\label{ass:path}
The known interpolation is jointly Borel and absolutely continuous in $r$ for every endpoint pair, with $I_0(x_0,x_1)=x_0$ and $I_1(x_0,x_1)=x_1$. It has a jointly Borel derivative version $Z_r$ and a known $c_Z>0$ such that
\begin{equation}\label{eq:schedule}
 \norm{Z_r(x_0,x_1)}\le c_Z(\norm{x_0}+\norm{x_1}+1)
\end{equation}
for almost every $r$ for each pair. Use the version set to zero where this envelope fails. Flow time in \eqref{eq:protocol} has a fixed law $\lambda$ on $[0,1]$, chosen by the learner.
\end{assumption}
\begin{remark}
Linear and sine--cosine paths qualify. The envelope bounds labels for the gate and the loss cover, and the exact endpoints are what generation needs.
\end{remark}

\begin{assumption}[Constrained neural class]\label{ass:network}
Fix $D\ge2$ affine layers with widths $n_0=d+1,n_1,\ldots,n_{D-1},n_D=d$. For input $h_0=(x,r)$, hidden states $h_j$, weight matrices $W_j$, bias vectors $b_j$, fixed coordinatewise activations $\phi_j$, and parameter vector $\theta$, the field is
\begin{equation}\label{eq:network}
 h_j=\phi_j(W_jh_{j-1}+b_j)\ (j<D),\quad
 v_\theta=W_Dh_{D-1}+b_D.
\end{equation}
Fixed coordinatewise activations are differentiable, $L_\phi$-Lipschitz, and bounded in $[-b,b]$ for finite $L_\phi,b$. The first hidden activation is nondecreasing. The nonempty parameter set $\Theta\subset\R^P$ is closed, with finite bounds $V_j,c_j$ satisfying $\norm{W_j}_{\mathrm{op}}\le V_j$ and $\norm{b_j}\le c_j$. Here $P$ counts the trainable scalar entries, each once. Fixed entries are allowed. Set $\V=\{v_\theta:\theta\in\Theta\}$, fixed across rounds.
\end{assumption}
Proposition~\ref{prop:architecture} derives compactness, parameter covers, a global output bound $B_v$, and a global spatial Lipschitz constant $L_x$:
\begin{equation}\label{eq:architecturebounds}
 B_v=V_Db\sqrt{n_{D-1}}+c_D,\quad
 L_x=V_D\prod_{j<D}(L_\phi V_j).
\end{equation}
\begin{remark}
These bounds give a parametric cover of the proxy losses, a well-posed learned flow, and bounded loss differences. The cover complexity is linear in $P$, as in \citet{sahoo2026}, and bounded activations exclude unclipped ReLU-type networks (Sections~\ref{sec:localrad} and~\ref{app:architecture}). No curvature of $\theta\mapsto L_t(v_\theta)$ is assumed.
\end{remark}

To state the fitting requirement, let $M>0$ be a deterministic threshold and $G_M$ its endpoint gate. Define the proxy loss $\ell_M$ and its bound $B$ by
\begin{align}
 G_M&=\ind\{\norm{X_0}\le M,\norm{X_1}\le M\},\nonumber\\
 \ell_M(v)&=G_M\norm{v(U)-Z}^2,\label{eq:gate}\\
 B&=[B_v+c_Z(2M+1)]^2>0.\nonumber
\end{align}
For a deterministic history $H$ containing $q$ rounds, its empirical proxy objective $\widehat L_{H,M}$ averages $\ell_M$ over all $mq$ examples, including rejected examples as zero summands.

\begin{assumption}[Empirical fitting interface]\label{ass:fit}
The fitting routine has an accuracy $\eopt\ge0$: every call returns a measurable $\widehat v\in\V$ using only its historical data and independent randomness, with
\begin{equation}\label{eq:fit}
 \widehat L_{H,M}(\widehat v)\le
 \inf_{v\in\V}\widehat L_{H,M}(v)+\eopt
\end{equation}
almost surely. The learner knows $T,m$ and valid $\kappa,c_Z,B_v$ envelopes for gating and normalization.
\end{assumption}
\begin{remark}
The accuracy $\eopt$ is a property of the routine, not an input of the learner, and it enters every bound additively as $T\eopt$. A routine with $\eopt=O(T^{-1/2})$ adds $O(\sqrt T)$. Assumption~\ref{ass:fit} is a standard interface that separates statistics from optimization, and accuracy guarantees for particular training routines are outside the scope of this work. 
\end{remark}

The target-law results add a design condition on the path, the base, and the flow-time law. For affine schedules $\alpha,\beta$ write $J=\alpha\beta'-\alpha'\beta$ for their Wronskian, and for a base density $p_\nu$ with $\sqrt{p_\nu}\in H^1(\R^d)$ write $I_\nu=4\int\nabla\sqrt{p_\nu}\,\nabla\sqrt{p_\nu}^{\top}$ for its Fisher information.

\begin{assumption}[Smoothing design]\label{ass:smooth}
(a) The interpolation is affine, $I_r(x_0,x_1)=\alpha(r)x_0+\beta(r)x_1$ with $\alpha,\beta\in C^1([0,1])$, $\alpha(0)=\beta(1)=1$, $\alpha(1)=\beta(0)=0$, so $Z_r=\alpha'(r)x_0+\beta'(r)x_1$ and \eqref{eq:schedule} holds with $c_Z=\max\{\norm{\alpha'}_\infty,\norm{\beta'}_\infty\}$. (b) The base has a density with $\sqrt{p_\nu}\in H^1(\R^d)$ and $\iota:=\lambda_{\max}(I_\nu)<\infty$. (c) Flow time has a fixed law $\lambda$ with $\alpha(r)>0$ for $\lambda$-almost every $r$ and
\begin{equation}\label{eq:moments}
 m_h:=\E_\lambda\frac{\beta(r)^2}{\alpha(r)^2}<\infty,\qquad
 m_k:=\E_\lambda\frac{J(r)^2}{\alpha(r)^2}<\infty.
\end{equation}

\end{assumption}
\begin{remark}\label{rem:smooth}
Condition (a) covers linear, sine--cosine, and delayed-target schedules, and (b) holds for $\nu=N(0,\sigma_0^2I_d)$ with $\iota=1/\sigma_0^2$. Together they make every input law a noisy image of the target, and (c) asks that this smoothing survive on average over flow time. Since $\alpha(r)\le\norm{\alpha'}_\infty(1-r)$, (c) excludes uniform flow time for every schedule (Proposition~\ref{prop:necessity}). It admits $\mathrm{Unif}[0,1-\eta_0]$ for $0<\eta_0<1$ (linear path: $m_k=1/\eta_0$), $\mathrm{Beta}(a,b)$ with $a>0,b>2$, and logit-normal laws \citep{esser2024} (Appendix~\ref{app:timelaws}).
\end{remark}

\paragraph{Predictive information.}
Before round $t$, candidate fields and selection probabilities depend only on earlier batches and seeds. The learner observes labels $Z$, not $b_t$.

\begin{remark}[Flow-time law in the relative-risk results]\label{rem:timelaw}
The proofs of Theorem~\ref{thm:main} and Proposition~\ref{thm:historical} use $\lambda$ only through the law of $(U,Z)$, so they hold for every $\lambda$, including the laws of Assumption~\ref{ass:smooth}, and one run can satisfy both main theorems.
\end{remark}

\section{Main Results}\label{sec:results}
\subsection{The Learner: DyProd-FM}\label{sec:learner}
DyProd-FM keeps fits with different starting times to balance estimation error against drift. The dyadic family $\mathcal D_T$ consists of the distinct nonempty intersections of binary-tree nodes on $[1,2^{\lceil\log_2T\rceil}]$ with $[1,T]$. Its root is $[1,T]$. At most $1+\lceil\log_2T\rceil$ lifetimes are active per round. Each fits the largest power-of-two prefix of its past and refits when that prefix doubles. Sleeping Prod \citep{gaillard2014} selects one field before the fresh batch arrives, so every deployed field belongs to $\V$.

\begin{figure}[!htb]
\refstepcounter{algorithmctr}\label{alg:nsfm}
\centerline{\textbf{Algorithm~\thealgorithmctr:} DyProd-FM}\vspace{3pt}
\small
\hrule\vspace{3pt}
\textbf{Given:} horizon $T$ and batch size $m$.\\
\textbf{Design parameters:} flow-time law $\lambda$, gate radius $M$ (default in Equation~\eqref{eq:gatechoice}), normalization $B=[B_v+c_Z(2M+1)]^2$, master rate $\eta=1/32$, and a training routine $\textsc{PrefixFM}$ for the network of Assumption~\ref{ass:network} (Algorithm~\ref{alg:fit}).
\begin{enumerate}[label=\arabic*:,leftmargin=1.6em,itemsep=1pt,topsep=3pt]
\item Form the dyadic lifetimes $\mathcal D_T$, and set $M_e=|\mathcal D_T|$ and $w_I=1/M_e$ for all $I\in\mathcal D_T$. For each start $a$, draw $\theta_a^{(0)}=\Pi_\Theta(g_a)$, $g_a\sim N(0,P^{-1}I_P)$, where $I_P$ is the identity matrix and $\Pi_\Theta$ is a measurable Euclidean projection (Appendix~\ref{sec:algremarks}).
\item \textbf{for} $t=1,\ldots,T$ \textbf{do}
\item\quad Process the root lifetime first. For each active $I=[a,b]\ni t$: if $t=a$, let $v_{I,t}$ be the most recent completed root fit, or $v_{\theta_1^{(0)}}$ when $t=1$. Otherwise $q=2^{\lfloor\log_2(t-a)\rfloor}$ and $v_{I,t}=\textsc{PrefixFM}(\widehat L_{[a,a+q-1],M})$, started from $\theta_a^{(0)}$ if $q=1$ and from the fit for $q/2$ otherwise, refitted only when $q$ doubles.
\item\quad Draw $I_t$ with $p_{I,t}=w_I/\sum_{K\ni t}w_K$ and deploy $v_t=v_{I_t,t}$.
\item\quad Receive $X_{1,t,i}\sim\mu_t$, draw $X_{0,t,i}\sim\nu$ and $r_{t,i}\sim\lambda$, form triples $\xi_{t,i}=(X_{0,t,i},X_{1,t,i},r_{t,i})$ ($i\le m$), and score $x_{I,t}=\frac1{Bm}\sum_i\ell_M(v_{I,t};\xi_{t,i})$ and $\bar x_t=\sum_{I\ni t}p_{I,t}x_{I,t}$.
\item\quad Update $w_I\gets w_I[1+\eta(\bar x_t-x_{I,t})]$ for active $I$. Inactive weights are unchanged.
\item \textbf{end for}
\end{enumerate}
\vspace{2pt}\hrule
\end{figure}
Appendix~\ref{sec:algorithm-details} specifies initialization and prefix training. Appendix~\ref{sec:resources} gives fitting, scoring, memory, and per-round costs conditional on the fitting interface.

\subsection{Theoretical Guarantees}\label{sec:theoretical-results}
Expected dynamic velocity regret $\E\Reg=\E\sum_{t=1}^T F_t(v_t)$ compares DyProd-FM with each round's best represented field. The bounds $\Fast$ and $\Targ$ use the drift measures of Section~\ref{sec:budgets}. Dividing by $T$ gives expected average excess velocity risk. Use Algorithm~\ref{alg:nsfm}'s rate $\eta=1/32$ and any $M>0$. Drift budgets enter only the analysis.

Let $M_e=|\mathcal D_T|<4T$, $\Lambda=\log(eM_e)$, $\zeta=B/(mT)^2$, and $\Gamma=1+\log(2K_M)$. Here $K_M$ is the size of a deterministic uniform $\zeta$-cover of proxy losses (Appendix~\ref{app:architecture}), used only in the analysis. Proposition~\ref{prop:tail} bounds the uniform original-to-proxy risk discrepancy by $\delta=\delta_M$, where
\begin{align}
 C_\ell&=2[(B_v+c_Z)^2+c_Z^2],\nonumber\\
 \delta_M&=2C_\ell(1+M^2+4\kappa^2)e^{-M^2/(4\kappa^2)}.\label{eq:delta}
\end{align}

\begin{theorem}[Relative-risk tracking]\label{thm:main}
Under the predictive protocol, Assumptions~\ref{ass:tails}--\ref{ass:fit}, and the caps in Equation~\eqref{eq:caps}, Algorithm~\ref{alg:nsfm} satisfies $\E\Reg\le\Fast$, where
\begin{align}
 \Fast={}&\min\big\{K_T(N_\omega+1),\ K_T+2\sqrt{10K_TT\Omega}\big\}\nonumber\\
 &+H_T\E\App+Te_f.\label{eq:fast}
\end{align}
The interval cost $K_T$, approximation multiplier $H_T$, residual $e_f$, and dyadic-cover size bound $D_T$ are
\begin{align}
 D_T&=2(1+\lceil\log_2T\rceil),\quad H_T=8\log(eT),\nonumber\\
 K_T&=D_T\left[44B\Lambda+\frac{120B\Gamma}{m}\log(eT)\right],\label{eq:constants}\\
 e_f&=4\eopt+9\zeta+20\delta.\nonumber
\end{align}
\end{theorem}
The learner adapts to the smaller cost of count $N_\omega$ and variation $\Omega$. Approximation costs $H_T\E\App$, while $Te_f$ collects fitting, covering, and gating errors. Section~\ref{sec:drift} outlines the proof, detailed in Appendix~\ref{app:regret}. For the variation branch with fixed positive $\Omega$, take $\eopt=O(T^{-1/2})$.

For target-law tracking, base smoothing gives sharper risk transfer. Define the base-translation and label-shift moduli $h(r)=\sqrt{\iota/8}\,|\beta(r)|/\alpha(r)$ and $k(r)=|J(r)|/\alpha(r)$, and the label threshold $M_z=c_Z(1+M)$. Using the moments $m_h,m_k$ of Assumption~\ref{ass:smooth}, set
\begin{align}
 \Lsm&=2\sqrt{m_k}+4\sqrt3\,(B_v+M_z)\sqrt{\iota m_h/8},\nonumber\\
 \esm&=16\sqrt3\,B_v\sqrt{\iota m_h/8}\,\sqrt{\delta_M/2},\nonumber\\
 c_2&=8\sqrt3\,B_v\E_\lambda[hk],\quad \bar c=99\Lsm^2+37c_2.\label{eq:smoothing-constants}
\end{align}
In Lemma~\ref{lem:smoothing}, $\Lsm$ controls smoothing, $\esm$ label tails, and $c_2$ the quadratic remainder. Their interval coefficient is $\bar c$.

\begin{theorem}[Target-law tracking]\label{thm:target}
Under the predictive protocol, Assumptions~\ref{ass:tails},~\ref{ass:network},~\ref{ass:fit},~\ref{ass:smooth}, and the caps in Equation~\eqref{eq:targetvariation}, Algorithm~\ref{alg:nsfm} satisfies $\E\Reg\le\Targ$, where
\begin{align}
 \Targ={}&\min\Big\{K_T(N_\mu+1)+H_T\E\App+Te_f,\nonumber\\
 &\quad 2K_T^++2(2\bar c)^{1/3}(K_T^+)^{2/3}T^{1/3}V_\mu^{2/3}\nonumber\\
 &\quad+37\esm TV_\mu+H_T^+\E\App+Te_f^+\Big\}.\label{eq:target}
\end{align}
The variation branch uses the interval cost, approximation multiplier, and per-round residual
\begin{align}
 H_T^+&=300\log(eT)+138,\nonumber\\
 K_T^+&=D_T\left[44B\Lambda+\frac{360B\Gamma}{m}\log(eT)\right],\label{eq:constants-target}\\
 e_f^+&=10\eopt+25\zeta+55\delta.\nonumber
\end{align}
\end{theorem}
The bound $\Targ$ adapts to count $N_\mu$ or Wasserstein movement $V_\mu$. Quadratic transfer gives the $T^{1/3}V_\mu^{2/3}$ cost and remainder $37\esm TV_\mu$, alongside approximation and fitting errors. Section~\ref{sec:architecture} outlines the proof, detailed in Appendix~\ref{app:targetproof}. For the variation branch with fixed positive $V_\mu$, take $\eopt=O(T^{-2/3})$.

The constants are valid at $T=m=1$, and both theorems apply to one run under their joint assumptions. Choosing
\begin{equation}\label{eq:gatechoice}
 M=4\kappa\sqrt{\log(emT)}
\end{equation}
gives the cumulative proxy bias
\begin{equation}\label{eq:tailrate}
 T\delta_M\le
 \frac{2C_\ell T[1+16\kappa^2\log(emT)+4\kappa^2]}{(emT)^4},
\end{equation}
and $\esm TV_\mu=O(V_\mu\sqrt{\log(emT)}/(m^2T))$. At fixed dimension, batch size, architecture, and regularity scales, $B,\Gamma,\Lambda,\Lsm$ grow only polylogarithmically in $T$, giving the dynamic rates in the abstract. A persistent approximation error $a_t=a>0$ contributes $H_TaT$ or $H_T^+aT$, depending on the branch. Appendix~\ref{app:lower} accounts for approximation when matching the lower bounds.

\paragraph{Lower bounds.} Appendix~\ref{app:lower} complements Theorems~\ref{thm:main} and~\ref{thm:target} with lower bounds for fixed classes and problem constants. Only the target sequence varies. Gaussian targets with at most $N<T$ mean changes force regret of order $N+1$ in both variations (Proposition~\ref{prop:lb-count}). A rare-mixture family, the block version of the two-point construction of \citet{sahoo2026}, forces $\sqrt{T\Omega/m}$ under uniform flow time for one fixed class, whose trainable slope lets its cutoff follow the far target component (Proposition~\ref{prop:lb-omega}). Gaussian targets with slowly moving means force $T^{1/3}V_\mu^{2/3}m^{-2/3}$ for every flow-time law of Assumption~\ref{ass:smooth} (Proposition~\ref{prop:lb-target}). Theorems~\ref{thm:main} and~\ref{thm:target} match these bounds up to polylogarithmic factors in $mT$, factors of $m$, and additive fitting costs.

\section{Proof Outline of the Main Results}\label{sec:analysis}

Our flow-matching analysis contributes fast estimation from heterogeneous histories, a risk-dependent second moment for proper selection, and quadratic transfer under target-law drift. Standard dyadic covers and Prod potentials assemble these estimates. Following Appendix~\ref{app:common}, we establish the shared results, then complete Theorems~\ref{thm:main} and~\ref{thm:target} as in Appendices~\ref{app:regret} and~\ref{app:targetproof}.

\subsection{Historical Estimation for Changing Targets}\label{sec:hist}
For a deterministic historical interval $H$ of $q$ rounds, let $n=mq$ be its number of examples. Define its average population loss $\bar L_H$, best common field $v_H^\star$, common-field approximation error $\bar a_H$, and excess loss $r_H$ by
\begin{align}
 \bar L_H(v)&=\frac1q\sum_{s\in H}L_s(v),\quad
 v_H^\star\in\argmin_{v\in\V}\bar L_H(v),\nonumber\\
 \bar a_H&=\frac1q\sum_{s\in H}Q_s(v_H^\star),\quad
 r_H(v)=\bar L_H(v)-\bar L_H(v_H^\star).\label{eq:historical}
\end{align}
One field must serve all historical laws, so $\bar a_H$ can exceed the average of the $a_s$ and remain positive even when every round is realizable. This drift-induced residual extends fixed-target estimation: pooling reduces sampling error but can increase the common comparator's cost.

Index the independent historical triples by $\chi_i$, $i=1,\ldots,n$, and define $g_i(v)=\ell_M(v;\chi_i)-\ell_M(v_H^\star;\chi_i)$. Lemma~\ref{lem:secondmoment} in Appendix~\ref{sec:shared-statistics} gives the key second-moment bound
\begin{equation}\label{eq:bernstein}
 \frac1n\sum_i\operatorname{Var}(g_i(v))
 \le8B\{r_H(v)+2\bar a_H\}.
\end{equation}
For $u,v\in\V$, the absolute difference of their gated squared residuals is at most $2\sqrt B\norm{u(U)-v(U)}$. Squaring and comparing both fields with the round-specific ideal velocity proves \eqref{eq:bernstein}. This approximate Bernstein condition in function space comes from the regression identity in Equation~\eqref{eq:regression}, not from parameter-space quadratic growth or a Polyak--\L{}ojasiewicz inequality. Its approximation term cannot be deleted for nonconvex classes (Appendix~\ref{app:statistics}).

\begin{proposition}[Historical estimation]\label{thm:historical}
Under Assumptions~\ref{ass:tails}--\ref{ass:fit}, fix a deterministic history of $q\ge1$ rounds, condition on $E$, and let $\widehat v$ satisfy \eqref{eq:fit}. Then, with $n=mq$,
\begin{equation}\label{eq:histfast}
 \E[r_H(\widehat v)\mid E]
 \le\frac{36B\Gamma}{mq}+2\bar a_H
 +2\eopt+5\zeta+10\delta,
\end{equation}
and the high-probability bound \eqref{eq:historicalhp} of Appendix~\ref{app:statistics} holds.
\end{proposition}
\paragraph{Proof of Proposition~\ref{thm:historical}.} Lemma~\ref{lem:secondmoment} makes Bernstein's deviation depend on excess risk. A union bound over the loss cover gives a uniform inequality that applies to the data-dependent fit. Assumption~\ref{ass:fit} bounds its empirical excess by $\eopt$. Absorbing the square-root risk term and integrating the tail yields Equation~\eqref{eq:histfast}. Appendix~\ref{sec:proof-historical} proves this extension of \citet{sahoo2026} to heterogeneous histories. The common comparator's residual $\bar a_H$ remains explicit while estimation scales as $1/(mq)$. We next show that selection preserves this fast rate.

\subsection{Aggregation over Dyadic Lifetimes}\label{sec:aggregation}
To control selection among historical fits, we use Appendix~\ref{sec:shared-aggregation}. Index lifetimes by $j$, with current field $v_{j,t}$, normalized batch score $x_{j,t}\in[0,1]$, and pre-batch selection probability $p_{j,t}$. Sums over $j$ or $k$ below run over active lifetimes. Set $\bar x_t=\sum_jp_{j,t}x_{j,t}$ and fix a lifetime $j=I$. With $d_{j,t}=\bar x_t-x_{j,t}$, weight conservation and $\log(1+u)\ge u-u^2$ for $|u|\le1/2$ yield the pathwise inequality
\begin{equation}\label{eq:master}
 \sum_{t\in I}d_{j,t}\le\frac{\log M_e}{\eta}
 +\eta\sum_{t\in I}d_{j,t}^2,
\end{equation}
proved pathwise in Lemma~\ref{lem:potential} by the second-order Prod potential argument \citep{gaillard2014}. The fast rate needs a population calculation. Conditional on the pre-round information and $E$, denote expectation over the new batch and selection by $\E_{\mathrm{new}}$. Jensen and the gated squared-loss identity give
\begin{align}
 \E_{\mathrm{new}}d_{j,t}^2
 \le\frac8B\bigg[\sum_kp_{k,t}Q_t(v_{k,t})
 +Q_t(v_{j,t})\bigg].\label{eq:mastervariance}
\end{align}
The first term is the selected field's expected absolute velocity risk. This flow-matching-specific estimate extends the standard Prod argument: one fixed rate preserves fast estimation while selecting a represented field. Let $S_I=\sum_{t\in I}\E[F_t(v_t)\mid E]$, $J_I=\sum_{t\in I}\E[F_t(v_{I,t})\mid E]$, and $\mathcal A_I=\sum_{t\in I}a_t$. Inserting \eqref{eq:mastervariance} into \eqref{eq:master} at rate $1/32$ gives
\begin{equation}\label{eq:absorb}
 S_I-J_I\le32B\log M_e+\tfrac14(S_I+J_I+2\mathcal A_I)
 +2|I|\delta,
\end{equation}
and absorbing $S_I/4$ gives $S_I\le\frac53J_I+\frac{128}3B\log M_e+\frac23\mathcal A_I+\frac83|I|\delta$ (Lemma~\ref{lem:fastmaster}), without averaging represented fields. Appendix~\ref{sec:shared-aggregation} gives both detailed proofs. The selection overhead is thus $O(B\log M_e)$ per lifetime.

It remains to bound $J_I$, the current risks of a lifetime's historical fits, under each drift measure.

\subsection{Relative-Risk Transfer and Proof of Theorem 1}\label{sec:drift}
We now transfer Proposition~\ref{thm:historical}'s historical excess to the current law. For a history $H=[a,a+q-1]$ with $q\ge1$ and $a+q\le t$, let $\Omega_{a:t}=\sum_{s=a}^{t-1}\omega_s$. Lemma~\ref{lem:drift}, proved in Appendix~\ref{sec:shared-statistics}, gives
\begin{align}
 \bar a_H&\le\frac1q\sum_{s\in H}a_s+2\Omega_{a:t},\label{eq:histdrift}\\
 F_t(\widehat v)&\le r_H(\widehat v)+2\Omega_{a:t}.\label{eq:currentdrift}
\end{align}
Both are first order in drift. Insert them into Proposition~\ref{thm:historical}. At age $k=t-a\ge1$, the prefix length $q\ge k/2$ makes the $1/q$ errors sum harmonically. Lemma~\ref{lem:base} (Appendix~\ref{sec:shared-aggregation}) gives the resulting lifetime bound on $J_I$, including approximation, fitting, and drift.

For a consecutive interval $J$, write $n=|J|$, let $\Omega_J$ sum its internal relative-risk increments, and put $\mathcal A_J=\sum_{t\in J}a_t$. Lemma~\ref{lem:dyadiccover} partitions $J$ into at most $D_T$ disjoint lifetimes. Combining their base bounds with Lemma~\ref{lem:fastmaster} gives

\begin{equation}\label{eq:intervalfast}
 \sum_{t\in J}\E[F_t(v_t)\mid E]
 \le K_T+H_T\mathcal A_J+ne_f+10(n-1)\Omega_J.
\end{equation}
\paragraph{Complete Theorem~\ref{thm:main}.} Splitting at relative-risk changes eliminates internal variation, giving $K_T(N_\omega+1)$. Alternatively, blocks of length at most $h\in\{1,\ldots,T\}$ cost $K_T\lceil T/h\rceil+10(h-1)\Omega$. Appendix~\ref{app:regret} chooses integer $h$ to bound this by $K_T+2\sqrt{10K_TT\Omega}$. These are proof-only partitions. Taking the smaller cost, adding $H_T\E\App+Te_f$, and integrating over $E$ gives $\E\Reg\le\Fast$, proving Theorem~\ref{thm:main}. Target-law tracking uses the same estimates with the sharper transfer below.

\subsection{Base Smoothing and Target-Law Transfer}\label{sec:link}
A target-law bound must convert Wasserstein movement into changes in velocity risk. For $u\in\V$, write $D_{s,t}(u)=L_t(u)-L_s(u)$ and define $\osc(f)=\sup_{v\in\V}f(v)-\inf_{v\in\V}f(v)$. Then $\omega_t=\frac12\osc(D_{t,t+1})$. We must control how $D_{s,t}$ separates an earlier fit from the current optimum at second order.

\paragraph{A target move is a base-noise shift plus a label shift.}
Let $(X_1,X_1')$ be a $W_2$-optimal coupling of $(\mu_s,\mu_t)$, let both rounds use the same $X_0\sim\nu$ and $r\sim\lambda$, and write $\Delta X=X_1'-X_1$. Writing $\alpha=\alpha(r)$, $\beta=\beta(r)$, and $J=J(r)=\alpha\beta'-\alpha'\beta$, the substitution $\tilde X_0=X_0+(\beta/\alpha)\Delta X$ gives, for $\alpha>0$, the exact identity
\begin{equation}\label{eq:baseshift}
\begin{split}
 &\big(I_r(X_0,X_1'),Z_r(X_0,X_1')\big)\\
 &\quad=\Big(I_r(\tilde X_0,X_1),\,Z_r(\tilde X_0,X_1)+\frac{J}{\alpha}\Delta X\Big).
\end{split}
\end{equation}
Round $t$ is round $s$ with its base noise translated by $(\beta/\alpha)\Delta X$ and its label translated by $(J/\alpha)\Delta X$ at the same input. The key target-law ingredient is that base regularity controls risk comparisons without spatially differencing fields or adding Lipschitz or norm-equivalence assumptions.

\begin{samepage}
\begin{lemma}[Second-order smoothing transfer]\label{lem:smoothing}
Under Assumptions~\ref{ass:tails},~\ref{ass:network} and~\ref{ass:smooth}, for all rounds $s,t$ and all $u,w\in\V$, with $W=W_2(\mu_s,\mu_t)$,
\begin{align*}
 |D_{s,t}(u)-D_{s,t}(w)|\le{}&\Lsm W\big(\norm{u-w}_s+\norm{u-w}_t\big)\\
 &+\esm W+c_2W^2.
\end{align*}
With bounded endpoints and $M_z$ an almost-sure label bound, the same holds with $\esm=0$.
\end{lemma}
\end{samepage}
Appendix~\ref{sec:proof-smoothing} proves this using Hellinger change of measure for the base, Cauchy--Schwarz for the label shift, and label tails for $\esm$. Lemma~\ref{lem:aqg} bounds squared prediction distance by excess risk plus approximation error. Lemma~\ref{lem:secondorder} absorbs the resulting square-root risk terms to give the next proposition.

\begin{proposition}[Link between the two variations]\label{prop:link}
Under the assumptions of Lemma~\ref{lem:smoothing}, let $\omega_{s,t}=\frac12\osc(D_{s,t})$ and let $v_s^\star$ minimize $L_s$ over $\V$. For all rounds $s,t$, with $W=W_2(\mu_s,\mu_t)$,
\begin{align}
 \omega_{s,t}&\le2B_v\Lsm W+\tfrac12(\esm W+c_2W^2),\label{eq:link-first}\\
 F_t(v^\star_s)&\le5\Lsm^2W^2+8(a_s+a_t)+2(\esm W+c_2W^2).\label{eq:link-second}
\end{align}
\end{proposition}
Appendix~\ref{sec:proof-link} gives the detailed proof. Equation~\eqref{eq:link-first} bounds $\Omega_T$ by a multiple of $V_T^\mu$ for bounded fields, yielding only $\sqrt{TV_\mu}$ through Theorem~\ref{thm:main}. Equation~\eqref{eq:link-second} bounds an earlier optimum's excess risk by $W^2$ terms plus approximation and smoothing remainders. Section~\ref{sec:architecture} uses this quadratic dependence to obtain $T^{1/3}V_\mu^{2/3}$.

\begin{proposition}[Sharpness of the smoothing modulus]\label{prop:necessity}
Let $\nu=N(0,\sigma_0^2I_d)$ with $\sigma_0>0$ and $I_d$ the identity matrix, let Assumption~\ref{ass:smooth}(a) hold, and let $q_{r,x_1}$ be the law of $I_r(X_0,x_1)$. Then for $\alpha(r)>0$,
\[
 \sup_{x_1\ne x_1'}\frac{\Hel(q_{r,x_1},q_{r,x_1'})}{\norm{x_1-x_1'}}=\frac{|\beta(r)|}{2\sqrt2\,\alpha(r)\sigma_0}=h(r),
\]
where $\Hel(P,Q)=[\frac12\int(\sqrt p-\sqrt q)^2]^{1/2}$ is the Hellinger distance for laws with densities $p,q$. For every affine schedule, $\E_{\mathrm{Unif}[0,1]}[\beta^2/\alpha^2]=\infty$.
\end{proposition}
Smoothing vanishes near the target endpoint, excluding uniform flow time from this argument. Appendix~\ref{sec:proof-necessity} proves Proposition~\ref{prop:necessity} and motivates Assumption~\ref{ass:smooth}'s flow-time condition. We now apply the quadratic transfer to historical fits.

\subsection{Proof of Theorem 2 from Target-Law Transfer}\label{sec:architecture}
Under Assumption~\ref{ass:smooth}, Lemma~\ref{lem:targettransfer} applies Lemma~\ref{lem:secondorder} across a historical prefix (Appendix~\ref{sec:support-theorem2}). For $H=[a,a+q-1]$ with $q\ge1$ and $a+q\le t$, set $V^\mu_{a:t}=\sum_{s=a}^{t-1}W_2(\mu_{s+1},\mu_s)$ and $\bar\rho_{a:t}=\esm V^\mu_{a:t}+c_2(V^\mu_{a:t})^2$. For every $v\in\V$, the resulting transfer is
\begin{align}
 \bar a_H\le{}&\frac9q\sum_{s\in H}a_s+8a_t+5\Lsm^2(V^\mu_{a:t})^2+2\bar\rho_{a:t},\label{eq:histdrift-target}\\
 F_t(v)\le{}&3r_H(v)+29\Lsm^2(V^\mu_{a:t})^2+\frac{36}q\sum_{s\in H}a_s\nonumber\\
 &+34a_t+10\bar\rho_{a:t}.\label{eq:currentdrift-target}
\end{align}
Proposition~\ref{thm:historical} and harmonic summation give the lifetime bound of Lemma~\ref{lem:basetarget}. Combining it with Lemmas~\ref{lem:fastmaster} and~\ref{lem:dyadiccover} gives Lemma~\ref{lem:intervaltarget}. For a consecutive interval $J$, set $n=|J|$, $\mathcal A_J=\sum_{t\in J}a_t$, and $V_J^\mu=\sum_{s,s+1\in J}W_2(\mu_{s+1},\mu_s)$. The result is

\begin{align}\label{eq:intervaltarget}
 \sum_{t\in J}\E[F_t(v_t)\mid E]
 \le{}&K_T^++H_T^+\mathcal A_J+ne_f^+\nonumber\\
 &+(n-1)\big[\bar c(V_J^\mu)^2+37\esm V_J^\mu\big].
\end{align}
Since $N_T^\omega\le N_T^\mu\le N_\mu$, Theorem~\ref{thm:main} gives the count branch under the same flow-time law (Remark~\ref{rem:timelaw}).

For the variation branch, maximal pieces satisfying $n(V_J^\mu)^2\le K_T^+/\bar c$ each cost at most $2K_T^+$ plus approximation, fitting, and smoothing terms. Appendix~\ref{sec:proof-theorem2}, following \citet{baby2019}, bounds their number by $1+(2\bar cTV_\mu^2/K_T^+)^{1/3}$. Summing \eqref{eq:intervaltarget} gives the variation branch, including $37\esm TV_\mu$. Taking the smaller branch and integrating over $E$ gives $\E\Reg\le\Targ$, proving Theorem~\ref{thm:target}.

\section{Conclusion}\label{sec:conclusion}
One proper, budget-free learner attains the rates $\sqrt{T\Omega}$ and $T^{1/3}V_\mu^{2/3}$ for predictive flow matching, which are minimax for fixed classes up to logarithmic factors. Flow matching supplies the curvature that convexity would otherwise provide. Nonaffine interpolations, target-law rates under uniform flow time, the dependence on $m$, and computational training guarantees remain open.

\clearpage
\section*{AI Use Statement}
The authors sought the help of large language models to improve the grammar and presentation of the manuscript. Additionally, LLMs were used for literature survey and formatting. The authors take complete responsibility for this work.

\clearpage
\input{appendices/setup}
\section{Extended Related Work}\label{app:related}
This appendix expands Section~\ref{sec:related}. It places each line of work relative to the three ingredients of this paper: drifting targets in flow matching, budget-free dynamic regret for a nonconvex class, and the two measures of nonstationarity.

\subsection{Flow Matching, Rectified Flow, and Stochastic Interpolants}
\citet{lipman2023} introduce flow matching, which learns a transport velocity by regressing on conditional velocities along a prescribed probability path. Rectified flow \citep{liu2023} uses linear interpolation between base and target samples and favors nearly straight transport paths. Stochastic interpolants \citep{albergo2023,albergo2025} formulate flows and diffusions through interpolation maps, possibly with additional latent noise, and characterize the drift by the same kind of conditional-mean regression. The guide of \citet{lipman2024guide} surveys paths, schedules, and implementations.

All of these frameworks use the regression identity that underlies Equation~\eqref{eq:regression}. Assumption~\ref{ass:path} covers the linear and sine--cosine paths, and the affine schedules of Assumption~\ref{ass:smooth} include both. Interpolants with additional latent noise, nonaffine interpolations, and couplings other than independent pairing are outside the present target-law analysis. The flow-time law is also a modeling choice in practice. \citet{esser2024} sample flow time from a logit-normal law when training rectified-flow transformers at scale, and such laws satisfy Assumption~\ref{ass:smooth}(c) (Appendix~\ref{app:timelaws}).

\subsection{Statistical Theory of Flow Matching}
\citet{benton2024} relate the sampling error of flow-matching methods to the accuracy of the learned velocity, under regularity of the learned flow. \citet{zhou2025} give an error analysis of flow matching for deep generative modeling with fixed targets, including bounded-support distributions and stronger score-regularity regimes. \citet{gaur2025} study the sample complexity of flow matching and separate approximation, estimation, and optimization errors. These results motivate the error decomposition used here, with regret, approximation error, and fitting error kept separate. We retain an explicit approximate-global fitting interface (Assumption~\ref{ass:fit}) rather than proving convergence of a generic neural optimizer.

\citet{sahoo2026} obtain order-optimal sample complexity for rectified flows. Their analysis rests on Bernstein-type variance control for the squared flow-matching loss, a localized Rademacher bound, and a rare-mixture lower bound. Our Assumptions~\ref{ass:tails}--\ref{ass:fit} follow the same setting. The difference is the protocol. These works study one fixed target and i.i.d.\ samples, whereas we study a drifting target, predictive deployment before each fresh batch, heterogeneous historical pools, and dynamic regret against the best represented field of every round. Our question is therefore different: how much excess current velocity error accumulates when deployment precedes each fresh batch and the target changes? Appendix~\ref{sec:localrad} compares our direct historical estimation with the localized analysis.

\subsection{Nonstationary Online Learning and Dynamic Regret}
Tracking the best expert under switches \citep{herbster1998} and dynamic regret for online convex optimization against moving comparators \citep{zinkevich2003} are the classical starting points. \citet{besbes2015} study non-stationary stochastic optimization under a budget on the temporal variation of the loss functions. Our relative velocity risk variation is a class-relative budget of the same functional type, measured by the half-oscillation of risk increments so that changes common to all represented fields are free. The target-law variation is different in kind: it is a budget on the targets themselves, not on the losses or on the comparators.

Strongly adaptive methods control regret on every interval. \citet{daniely2015} use geometric covering intervals, and \citet{jun2017} improve strongly adaptive rates through coin betting. \citet{zhang2018} show how interval guarantees yield dynamic regret by partitioning the horizon according to functional variation. Our dyadic lifetimes follow the geometric-covering idea, and our analysis also partitions the horizon only in the proof. \citet{wei2016} track the best expert in non-stationary stochastic environments and show how statistical variance interacts with switching and variation budgets. This stochastic-expert view is close to our fresh-batch protocol, in which one batch scores every active candidate. We adapt these ingredients to stochastic population risks of data-dependent historical fits, and we select a single field at every round so that the deployed field stays in the represented class.

The master of DyProd-FM is the Prod algorithm with sleeping experts. \citet{gaillard2014} prove second-order regret bounds in terms of excess losses for multiple-rate Prod and give the sleeping-expert reduction. With one rate, the selection probabilities are normalized weights (Appendix~\ref{sec:algremarks}). Our contribution at this step is the population analysis: the squared flow-matching loss bounds the second moment of the score differences by velocity risks, which turns the potential inequality into the fast master bound of Lemma~\ref{lem:fastmaster}.

\subsection{Fast Dynamic Regret and Total-Variation Forecasting}
\citet{baby2019} study online forecasting of total-variation-bounded sequences under squared loss and establish the minimax rate $n^{1/3}C_n^{2/3}$, with a matching lower bound and a bin-counting argument. \citet{baby2021} obtain optimal dynamic regret for exp-concave losses, and \citet{baby2022} obtain such rates with proper learners for strongly convex losses and beyond. Their budgets are path variations of comparator sequences in convex decision sets. Their optimality conclusions therefore do not directly identify the minimax rate for our relative-risk variation and endpoint-feedback model.

Theorem~\ref{thm:target} attains the exponents $T^{1/3}V_\mu^{2/3}$ of total-variation forecasting, and its proof uses the bin-counting partition (Appendix~\ref{app:targetproof}). The decision set, however, is a nonconvex neural class, and the learner is proper. The curvature that the convex theory obtains from strong convexity or exp-concavity comes here from the flow-matching regression identity and from base smoothing (Lemma~\ref{lem:smoothing}).

\subsection{Fast Statistical Rates and Localization}
Variance control by excess risk underlies fast rates and localized complexity methods \citep{bartlett2005,srebro2010,blm2013}. In a convex class, squared loss gives an exact Bernstein condition. In a nonconvex represented class it does not, because the approximation error remains in the comparison between risk and distance (Lemma~\ref{lem:secondmoment}). We use one-sided Bernstein inequalities for independent, non-identically distributed observations together with a deterministic cover of the class, which gives explicit constants for the expectation bounds used online (Appendix~\ref{app:statistics}). For this parametric class the localized route gives the same order (Appendix~\ref{sec:localrad}).

\subsection{Minimax Lower Bounds}
Two-point and multiple-hypothesis reductions are the standard tools of nonparametric minimax theory \citep{tsybakov2009}. \citet{sahoo2026} use a rare-mixture pair to show that their fixed-target sample complexity cannot be improved. Our lower bounds apply such reductions block by block in time (Appendix~\ref{app:lower}). Every hard instance keeps the class, the envelope, and the flow-time law fixed, and only the target sequence varies with $T$ and the budget. In the rare-mixture construction a trainable slope lets one fixed class follow a far target component, whose location grows like $\sqrt{\log(1/\eta)}$ as in \citet{sahoo2026}. This rules out constructions that shrink the hard class as the budget changes, which would change the constants that the upper bounds depend on.

%
\supppage

\section{Algorithm and Consequences of Assumptions}\label{app:consequences}
\subsection{Notation}\label{sec:notation}
\begin{longtable}{@{}p{0.2\textwidth}p{0.75\textwidth}@{}}
\toprule
Symbol & Meaning in Theorems~\ref{thm:main}--\ref{thm:target}\\
\midrule
\endfirsthead
\toprule Symbol & Meaning\\\midrule\endhead
$T,m,d$ & Horizon, triples per round, dimension.\\
$\nu,\mu_t,\lambda$ & Base law, round-$t$ target law, flow-time law.\\
$\V,B_v,L_x$ & Represented class, its output bound, and its spatial Lipschitz constant (Equation~\eqref{eq:architecturebounds}).\\
$Q_t,a_t,F_t$ & Absolute velocity risk, class approximation error, excess velocity risk (Equation~\eqref{eq:excess}).\\
$\Reg,\App$ & Dynamic velocity regret and cumulative approximation error (Equation~\eqref{eq:regret}).\\
$\Omega,N_\omega$ & Caps on the relative velocity risk variation and its change count (Equation~\eqref{eq:caps}).\\
$V_\mu,N_\mu$ & Caps on the target-law variation and its change count (Equation~\eqref{eq:targetvariation}).\\
$\kappa,c_Z$ & Endpoint and label envelopes (Assumptions~\ref{ass:tails}--\ref{ass:path}).\\
$\eopt$ & Accuracy of the fitting routine (Assumption~\ref{ass:fit}).\\
$\iota,m_h,m_k$ & Base Fisher-information bound and flow-time moments (Equation~\eqref{eq:moments}).\\
$M,B,\delta$ & Gate radius, proxy loss bound, original-to-proxy discrepancy (Equations~\eqref{eq:gate}, \eqref{eq:delta}).\\
$\zeta,K_M,\Gamma$ & Cover accuracy $B/(mT)^2$, cover size, and $\Gamma=1+\log(2K_M)$.\\
$M_e,\Lambda$ & Number of lifetimes and $\log(eM_e)$.\\
$D_T,K_T,H_T,e_f$ & Constants of Theorem~\ref{thm:main} (Equation~\eqref{eq:constants}).\\
$\Lsm,\esm,c_2,\bar c$ & Smoothing constants of Theorem~\ref{thm:target} (Equation~\eqref{eq:smoothing-constants}).\\
$K_T^+,H_T^+,e_f^+$ & Remaining constants of Theorem~\ref{thm:target} (Equation~\eqref{eq:constants-target}).\\
$\Fast,\Targ$ & The bounds in Equations~\eqref{eq:fast} and~\eqref{eq:target}.\\
\bottomrule
\end{longtable}
The same letter can carry a distinct role only with its specified index and type: for example, $b_t(x,r)$ is a vector field whereas $b_j$ in a network is a finite-dimensional layer bias. Empty sums equal zero and empty products equal one. All logarithms without a subscript are natural logarithms. A singleton interval has no internal variation edges.

\subsection{Algorithm and Implementation Details}\label{sec:algorithm-details}
\medskip
\noindent\fbox{\begin{minipage}{0.97\textwidth}
\refstepcounter{algorithmctr}\label{alg:fit}
\centerline{\textbf{Algorithm~\thealgorithmctr:} $\textsc{PrefixFM}$, gated flow-matching training on a historical prefix}
\vspace{4pt}
\textbf{Given:} the stored triples $\xi_{s,i}=(X_{0,s,i},X_{1,s,i},r_{s,i})$ of the prefix $H=[a,a+q-1]$, $i\le m$, formed in line~5 of Algorithm~\ref{alg:nsfm}, and the gate radius $M$.\\
\textbf{Design parameters:} the network of Assumption~\ref{ass:network}, an optimizer such as stochastic gradient descent or Adam, a number of iterations $N$, and the initial parameters $\theta_{\mathrm{start}}$: $\theta_a^{(0)}$ of Algorithm~\ref{alg:nsfm} for the prefix $[a,a]$, and the completed fit for $[a,a+q/2-1]$ for the prefix $[a,a+q-1]$ with $q\ge2$ (Appendix~\ref{sec:algremarks}).
\begin{enumerate}[label=\arabic*:,leftmargin=1.8em,itemsep=1pt,topsep=3pt]
\item For every stored triple, form $U_{s,i}=(I_{r_{s,i}}(X_{0,s,i},X_{1,s,i}),r_{s,i})$, $Z_{s,i}=Z_{r_{s,i}}(X_{0,s,i},X_{1,s,i})$, and $G_{s,i}=\ind\{\norm{X_{0,s,i}}\le M,\norm{X_{1,s,i}}\le M\}$.
\item Set $\theta\gets\theta_{\mathrm{start}}$ and reset the optimizer state.
\item \textbf{for} $j=1,\ldots,N$ \textbf{do}
\item\quad Update $\theta$ to decrease $\widehat L_{H,M}(\theta)=\frac1{mq}\sum_{s\in H}\sum_{i\le m}G_{s,i}\norm{v_\theta(U_{s,i})-Z_{s,i}}^2$, and project $\theta$ onto $\Theta$.
\item \textbf{end for}
\item \textbf{return} $v_\theta$.
\end{enumerate}
\end{minipage}}
\medskip

Algorithm~\ref{alg:fit} is the standard flow-matching training loop, as in \citet{sahoo2026}, applied to the gated loss of one historical prefix. Assumption~\ref{ass:fit} asks only that the returned network be within $\eopt$ of the minimum of $\widehat L_{H,M}$ over $\V$. The optimizer and $N$ determine $\eopt$, which the analysis treats as given. Accuracy guarantees for particular training routines are outside the scope of this work.

\paragraph{Flow-time laws.} The flow time of every triple is drawn from $\lambda$ when the triple is formed. Theorem~\ref{thm:main} uses $\lambda=\mathrm{Unif}[0,1]$, the standard choice. Theorem~\ref{thm:target} uses any law satisfying Assumption~\ref{ass:smooth}(c), for example $\mathrm{Unif}[0,1-\eta_0]$, $\mathrm{Beta}(a,b)$ with $b>2$, or a logit-normal law \citep{esser2024}. By Remark~\ref{rem:timelaw}, Theorem~\ref{thm:main} also holds under the laws of Assumption~\ref{ass:smooth}, so with such a law one run satisfies both theorems.
\subsubsection{Initialization and Predictability}\label{sec:algremarks}
Each expert is a dyadic lifetime with its current fitted field. Algorithm~\ref{alg:nsfm} uses the design parameters of Section~\ref{sec:learner}. The fitting accuracy $\eopt$ is determined by the training routine and is not an input.

\paragraph{Initialization.} For the compact feasible parameter set $\Theta\subset\R^P$ of Assumption~\ref{ass:network}, draw independently across starts and independently of the observations and selection seeds
\begin{equation}\label{eq:init-law}
 g_a\sim N(0,P^{-1}I_P),\qquad \theta_a^{(0)}=\Pi_\Theta(g_a),\qquad \Pi_\Theta(g)=\operatorname{lexmin}\argmin_{\theta\in\Theta}\norm{\theta-g}^2 .
\end{equation}
Order trainable coordinates by layer, with row-major weights followed by biases, and retain prescribed fixed entries. Compactness and the lexicographic tie rule define a deterministic measurable projection, so every initial field belongs to $\V$.

\paragraph{Prefixes and warm starts.} Process the root first. At age zero, a lifetime uses the latest completed root fit, or the fixed startup field $v_{\theta_1^{(0)}}$ if none exists. At age $k\ge1$, it uses its own prefix of length $q=2^{\lfloor\log_2k\rfloor}$: the $q=1$ fit starts from $\theta_a^{(0)}$, and each longer fit starts from the completed $q/2$ fit with optimizer buffers and step counter reset. Lifetimes with the same start share initialization and fits. Nested prefixes use only their own data. The root fallback never initializes a fit. The analysis depends only on the empirical gap in Assumption~\ref{ass:fit}, so projected shrink-and-perturb or fresh starts are also valid when they use only prefix data and independent randomness \citep{ash2020}.

\paragraph{Predictability and weights.} Every prefix ends by $t-1$, refits occur only when its length doubles, and scoring and weight updates use the pre-batch fields. The root is always active. All conceptual weights start at $1/M_e$ and receive no new mass on waking. Lemma~\ref{lem:potential} ensures positivity, and the common rate $\eta=1/32$ makes Prod probabilities normalized weights \citep{gaillard2014}. The predictable fallback satisfies $0\le F_a(v)\le B+\delta$ by Equation~\eqref{eq:proxybias}, which is all the lifetime proof needs at age zero. These rules specify initialization and data use. Optimizer convergence is covered only by Assumption~\ref{ass:fit}.

\subsection{Probability Model and Regression Identity}\label{app:setup}
\subsubsection{Conditioning and Predictability}
The environment takes values in a standard Borel space, and the target distribution is a jointly measurable probability kernel of $(E,t)$. The base is fixed. Given the environment, generate all triples with the product law in \eqref{eq:protocol}. Independent fitting and selection seeds may be sampled on an additional product probability space. This construction permits a random environment while ensuring that the targets cannot adapt to sample noise or the learner's choices.

For analysis, fix $E$. Let $\mathcal F_{t-1}$ contain all earlier batches, earlier selection seeds, and the fitting randomness used to form the round-$t$ candidates. Conditional on $E$ and $\mathcal F_{t-1}$, candidate fields and probabilities are fixed. The new selection seed and new batch are independent of this information and of each other. If $f_t(v)$ is any integrable population functional, proper random selection consequently satisfies
\begin{equation}\label{eq:conditional-selection}
 \E[f_t(v_t)\mid E,\mathcal F_{t-1}]=\sum_{j\ \mathrm{active}}p_{j,t}f_t(v_{j,t}).
\end{equation}
For a current empirical proxy score, the same conditioning gives its population mean $L_{t,M}(v_{j,t})/B$. The fields may depend arbitrarily on their earlier training observations, while the evaluation batch is fresh. Conditioning additionally on the selected index is unnecessary and would obscure why every active model can be compared on one batch.

A historical prefix attached to a deterministic lifetime and a deterministic age has deterministic endpoints. Its fit is random, but the concentration theorem is uniform in the candidate parameter after the history is fixed. We do not apply a fixed-sample theorem to a stopping-time history. Population-dependent partitions used later are deterministic once $E$ is fixed. These are the exact reasons that the historical and online arguments are compatible.

\subsubsection{Finiteness, Measurability, and the Regression Identity}
Assumption~\ref{ass:tails} and Jensen imply finite endpoint second moments. Write $S=\norm{X_0}+\norm{X_1}$ in this paragraph. Assumption~\ref{ass:path}, or Assumption~\ref{ass:smooth}(a), yields $\norm Z\le c_Z(S+1)$, and Proposition~\ref{prop:architecture} gives $\norm{v_\theta}\le B_v$. Hence $[B_v+c_Z(S+1)]^2$ is an integrable common envelope for every original loss. Conditional expectation is well defined in $L^2$, and Jensen gives $\E_t\norm{b_t(U)}^2\le\E_t\norm Z^2$.

The standard Borel construction permits disintegration to obtain a jointly measurable version of the conditional mean as a function of the environment and input. To verify Equation~\eqref{eq:regression}, write
\begin{equation}
 v(U)-Z=[v(U)-b_t(U)]+[b_t(U)-Z].
\end{equation}
Expand its squared norm. The cross term has conditional expectation zero because $v(U)-b_t(U)$ is measurable in $U$ and $\E_t[b_t(U)-Z\mid U]=0$. Both factors are square integrable, so this orthogonality is legitimate. Taking expectations gives Equation~\eqref{eq:regression}, and subtracting the minimum gives Equation~\eqref{eq:excess} and summing gives Equation~\eqref{eq:regret}.

For fixed data, the empirical proxy objective is continuous in $\theta$. For a fixed environment, dominated convergence gives continuity of $L_t(v_\theta)$ and $Q_t(v_\theta)$. Compactness therefore gives all stated minima. A countable dense subset of $\Theta$ yields the same extrema of continuous functions, proving measurability of minimum risks and the suprema defining $\omega_t$. Measurable minimizers follow from the measurable maximum theorem applied to the compact-valued argmin correspondence \citep{aliprantis2006}. The existence of a measurable optimizer is a mathematical statement, not a computational complexity bound.

\subsubsection{Risk Oscillation and an Offset Example}
For a bounded real function $f$ on $\V$, define $\osc(f)=\sup_{v\in\V}f(v)-\inf_{v\in\V}f(v)$. Then $2\omega_t=\osc(L_{t+1}-L_t)$. The elementary inequality $\osc(f+g)\le\osc(f)+\osc(g)$ proves that for $s<t$,
\begin{equation}\label{eq:oscillation-telescope}
 \osc(L_t-L_s)\le2\sum_{j=s}^{t-1}\omega_j.
\end{equation}
Moreover, adding any constant to $f$ leaves its oscillation unchanged. In fact, $\frac12\osc(f)=\inf_{c\in\R}\sup_{v\in\V}|f(v)-c|$, since the midpoint of its range attains the infimum. Thus the budget ignores precisely the common component of a risk increment. This quotient interpretation is elementary and is not a new seminorm construction.

For a concrete flow-matching example, let $d=1$, take a centered Gaussian base, and use linear interpolation. Let the represented class be constant fields $v_c(x,r)=c$ for $|c|\le R$ with $R>0$. This is an admissible constrained network subclass with zero output weights and trainable bounded output bias. Let the target alternate between $N(0,1)$ and $N(0,4)$. The observed label is $Z=X_1-X_0$ and has mean zero. Consequently $L_t(v_c)=c^2+\E_t Z^2$, so every relative risk is unchanged and $\omega_t=0$, whereas $\sup_c|L_{t+1}(v_c)-L_t(v_c)|=3$ at each alternation. The endpoints satisfy a common finite sub-Gaussian envelope. This example illustrates what the regret budget measures. It does not assert that the constant-field class accurately generates both targets.

\subsection{Neural-Class Consequences and Covering Numbers}\label{app:architecture}
\subsubsection{Compactness and Network Bounds}
\begin{proposition}\label{prop:architecture}
Under Assumption~\ref{ass:network}, the parameter set is compact and the field is globally bounded by $B_v$ and globally spatially $L_x$-Lipschitz, with the constants in Equation~\eqref{eq:architecturebounds}. It is differentiable in its input. For each finite radius $R$, define $k_j=L_\phi$ for $j<D$ and $k_D=1$, and define hidden-state envelopes by $H_0(R)=\sqrt{R^2+1}$ and $H_j(R)=b\sqrt{n_j}$ for $1\le j<D$. Then a parameter-sensitivity constant is
\begin{equation}\label{eq:paramconstant}
 C_R=\sum_{j=1}^D k_j[H_{j-1}(R)+1]\prod_{i=j+1}^D(k_iV_i),
\end{equation}
and it satisfies
\begin{equation}\label{eq:paramsensitivity}
 \sup_{\norm{x}\le R,\ r\in[0,1]}\norm{v_\theta(x,r)-v_{\theta'}(x,r)}
 \le C_R\norm{\theta-\theta'}.
\end{equation}
A valid radius containing the parameter set is
\begin{equation}\label{eq:paramradius}
 R_\Theta^2=\sum_{j=1}^D\big[\min(n_j,n_{j-1})V_j^2+c_j^2\big].
\end{equation}
\end{proposition}
\begin{proof}
Every matrix entry is bounded by the operator-norm envelope and every bias coordinate by its Euclidean envelope. Closedness and finite dimension therefore imply compactness. More quantitatively, the number of nonzero singular values of $W_j$ is at most $\min(n_j,n_{j-1})$, so $\norm{W_j}_F^2\le\min(n_j,n_{j-1})V_j^2$. Sum these bounds and the squared bias bounds. Removing fixed entries can only reduce the trainable-vector norm, which proves Equation~\eqref{eq:paramradius}.

Each coordinate of the final hidden vector is bounded by $b$, hence $\norm{h_{D-1}}\le b\sqrt{n_{D-1}}$. Applying the output affine map gives $\norm{v_\theta}\le V_Db\sqrt{n_{D-1}}+c_D=B_v$. For two inputs with the same flow time, every hidden layer multiplies their distance by at most $L_\phi V_j$ and the output layer by at most $V_D$. Composing gives $L_x$. The same reasoning applies to the full input $(x,r)$. Composition of differentiable maps gives differentiability, and continuous derivatives are not needed for these estimates.

To prove parameter sensitivity, compare two networks at the same input. Let $d_j=\norm{h_j-h'_j}$ for hidden layers and let $d_D$ denote the final output difference. Adding and subtracting $W_jh'_{j-1}+b_j$ inside the layer difference gives
\begin{align}
 d_j&\le k_j\norm{W_j(h_{j-1}-h'_{j-1})+(W_j-W'_j)h'_{j-1}+b_j-b'_j}\nonumber\\
 &\le k_jV_jd_{j-1}
 +k_j\big[\norm{W_j-W'_j}_{\mathrm{op}}H_{j-1}(R)+\norm{b_j-b'_j}\big].\label{eq:layerrecursion}
\end{align}
The first step for $j=D$ uses $k_D=1$. Each matrix operator-norm difference is at most its Frobenius difference, which in turn is at most the norm of the vector of all changed entries, and the same is true of each bias block. Both are at most $\norm{\theta-\theta'}$. Since $d_0=0$, expand the recursion: the perturbation introduced at layer $j$ is multiplied by the Lipschitz constants of layers $j+1,\ldots,D$. Summing these contributions is exactly Equation~\eqref{eq:paramconstant}. This proves \eqref{eq:paramsensitivity}.
\end{proof}
The product-sum expression is valid if a weight bound is zero, if a product equals one, and if different layers have different bounds. It never divides by $L_\phi V-1$. For a fixed architecture, $C_R=O(1+R)$. The first-layer monotonicity retained in Assumption~\ref{ass:network} is not used in this proof.

\subsubsection{Gated State Radius and Joint Loss Cover}
Absolute continuity of the interpolation and Assumption~\ref{ass:path} imply, for every endpoint pair,
\begin{equation}\label{eq:stateradiusproof}
 \norm{I_r(x_0,x_1)}
 \le\norm{x_0}+\int_0^r\norm{Z_s(x_0,x_1)}\,ds
 \le\norm{x_0}+c_Z(\norm{x_0}+\norm{x_1}+1).
\end{equation}
On the gate, define the resulting spatial radius as $R_M=M+c_Z(2M+1)$. A linear-endpoint interpolation $I_r=\alpha(r)x_0+\beta(r)x_1$ permits the sharper radius $M\sup_r(|\alpha(r)|+|\beta(r)|)$, where $\alpha,\beta$ are its scalar schedules. The radius need not equal $M$: the sine--cosine path can have radius $\sqrt2M$.

\begin{corollary}\label{cor:cover}
Let $L_M=2C_{R_M}\sqrt B$ be the parameter-Lipschitz constant for the proxy loss. Write $N(\varepsilon,\mathcal C,\norm{\cdot}_\infty)$ for the covering number of a function class $\mathcal C$ at accuracy $\varepsilon$ in supremum norm over all endpoint-time triples. For any accuracy $\varepsilon>0$, the proxy-loss class has a uniform cover with centers in $\V$ and log cardinality at most
\begin{equation}\label{eq:entropy}
 \log N(\varepsilon,\{\ell_M(v_\theta):\theta\in\Theta\},\norm{\cdot}_\infty)
 \le P\log\left(1+\frac{2R_\Theta L_M}{\varepsilon}\right).
\end{equation}
If the loss class is constant, one center suffices. The same cover approximates proxy excess losses against any one fixed field.
\end{corollary}
\begin{proof}
For represented predictions $u,v$, the difference-of-squares identity on the gate gives
\begin{align}
 \abs{\norm{v-Z}^2-\norm{u-Z}^2}
 &=\abs{\langle v-u,v+u-2Z\rangle}\nonumber\\
 &\le\norm{v-u}(\norm{v-Z}+\norm{u-Z})
 \le2\sqrt B\norm{v-u}.
\end{align}
Outside the gate both losses are zero. Inequality~\eqref{eq:paramsensitivity}, with radius $R_M$, makes this difference at most $L_M\norm{\theta-\theta'}$ for every triple. The Borel derivative version specified in Assumption~\ref{ass:path} satisfies its envelope everywhere, so the supremum assertion is literal and not restricted to a high-probability set of samples.

For a parameter-net radius $\gamma>0$, choose a maximal $\gamma$-separated set of points in $\Theta$. The open radius-$\gamma/2$ balls around them are disjoint and contained in a radius-$(R_\Theta+\gamma/2)$ ball in $\R^P$. Comparing volumes bounds their number by $(1+2R_\Theta/\gamma)^P$. Maximality makes this a $\gamma$-net. Take $\gamma=\varepsilon/L_M$ if $L_M>0$. If $L_M=0$, every parameter has the same loss. This proves \eqref{eq:entropy}. Subtracting a fixed comparator does not change the covering distances. A supremum-norm cover is also a cover in every empirical $L^2$ norm, under heterogeneous as well as identically distributed samples.
\end{proof}
The cover is for the complete vector-output network, not separate scalar-coordinate networks. The parameter count already contains the output dimension. Dimension may nevertheless enter $B_v$, $R_\Theta$, $L_x$, and the endpoint envelope. These constants are not all logarithmic in architecture size.

\subsubsection{Learned Dynamics and Approximation Boundaries}
For a represented field, global input Lipschitzness and continuity imply existence and uniqueness of the learned ODE on $[0,1]$. Its bounded speed gives
\begin{equation}\label{eq:boundedtrajectory}
 \sup_{r\in[0,1]}\norm{Y_r}\le\norm{Y_0}+B_v.
\end{equation}
There is thus no finite-time escape. A measurable solution map follows from Picard iteration, and \eqref{eq:boundedtrajectory} gives finite second moments when the base has finite second moment. None of these statements grants uniqueness for the ideal field.

Bounded outputs do impose an approximation restriction. Write $x_+=\max\{x,0\}$ for the positive part. The reverse triangle inequality gives, for every represented $v$,
\begin{equation}\label{eq:approx-floor}
 Q_t(v)\ge\E_t(\norm{b_t(U)}-B_v)_+^2,
\end{equation}
The right side may be positive even for infinite data. Thus Assumption~\ref{ass:network} alone does not imply realizability. Two further limits of the class are worth stating. The cover complexity in \eqref{eq:entropy} is linear in the number $P$ of trainable parameters, as is the localized Rademacher bound of \citet{sahoo2026} (Appendix~\ref{sec:localrad}), so the guarantees can be loose for heavily overparameterized networks. Bounded hidden activations exclude unclipped ReLU, GELU, and SiLU networks, a restriction shared with \citet{sahoo2026}. If widths or norm bounds increase with $T$ to reduce this floor, their effects on the cover and the flow-stability constant must be included. Parameter sensitivity is local in the input radius: the family $v_w(x)=\tanh(wx)$, $w\in[0,1]$, has $\sup_x|v_\varepsilon(x)-v_0(x)|=1$ for every $\varepsilon>0$, precluding a uniform global parameter-Lipschitz bound. The gated radius is therefore necessary in the entropy calculation.

\subsection{Endpoint Gating and Original-Risk Transfer}\label{app:proxy}
\begin{proposition}\label{prop:tail}
Under Assumptions~\ref{ass:tails}--\ref{ass:network}, the proxy in Equation~\eqref{eq:gate} satisfies $0\le\ell_M(v)\le B$. Its population version $L_{t,M}(v)=\E_t\ell_M(v)$ obeys, for every round and represented field,
\begin{equation}\label{eq:proxybias}
 0\le L_t(v)-L_{t,M}(v)\le\delta_M,
\end{equation}
with $\delta_M$ given by Equation~\eqref{eq:delta}. The calibration in Equation~\eqref{eq:gatechoice} gives \eqref{eq:tailrate}.
\end{proposition}
\begin{proof}
Fix the environment and round. Let $S=\norm{X_0}+\norm{X_1}$. Since the represented output is bounded,
\begin{align}
 \norm{v(U)-Z}^2
 &\le[B_v+c_Z+c_ZS]^2\nonumber\\
 &\le2(B_v+c_Z)^2+2c_Z^2S^2\le C_\ell(1+S^2).
\end{align}
On the gate $S\le2M$, giving the proxy bound $B$. To bound the omitted expected loss, use $S^2\le2(\norm{X_0}^2+\norm{X_1}^2)$ and Cauchy--Schwarz:
\begin{align}
 \E_t e^{S^2/(4\kappa^2)}
 &\le\E_t e^{(\norm{X_0}^2+\norm{X_1}^2)/(2\kappa^2)}\nonumber\\
 &\le\big(\E_te^{\norm{X_0}^2/\kappa^2}\big)^{1/2}
       \big(\E_te^{\norm{X_1}^2/\kappa^2}\big)^{1/2}\le2.
\end{align}
This step does not require independence of the endpoints. Markov's inequality gives $\Prob(S^2>u\mid E)\le2e^{-u/(4\kappa^2)}$. Rejection implies $S>M$. For any nonnegative random variable $Y$ and threshold $a$, tail integration gives
\begin{equation}
 \E[(1+Y)\ind\{Y>a\}]=(1+a)\Prob(Y>a)+\int_a^\infty\Prob(Y>u)\,du.
\end{equation}
Apply it with $Y=S^2$ and $a=M^2$. The omitted loss is at most
\begin{align}
 C_\ell\E_t[(1+S^2)\ind\{G_M=0\}]
 &\le2C_\ell\left[(1+M^2)e^{-M^2/(4\kappa^2)}
 +\int_{M^2}^\infty e^{-u/(4\kappa^2)}du\right]\nonumber\\
 &=2C_\ell(1+M^2+4\kappa^2)e^{-M^2/(4\kappa^2)}.
\end{align}
The loss is nonnegative, so the discrepancy also has the correct lower sign. Substituting $M^2=16\kappa^2\log(emT)$ makes the exponential $(emT)^{-4}$, which proves \eqref{eq:tailrate}.
\end{proof}

\subsubsection{Consequences Used by the Statistical Proof}
Define the nonnegative bias function $\beta_t(v)=L_t(v)-L_{t,M}(v)$. Inequality~\eqref{eq:proxybias} states $\beta_t(v)\in[0,\delta]$. Thus a proxy difference and its original counterpart differ by at most $\delta$, and the later estimates conservatively use $2\delta$. If proxy losses of two fields differ uniformly by at most $\zeta$, their original risks differ by at most $\zeta+2\delta$. These inequalities also hold after averaging over any deterministic historical interval. They are uniform over $v$, so they apply to data-dependent fits.

The gate is shared by all models, independent of their predictions. A rejected example remains in the denominator of the empirical average. This makes the objective exactly the empirical mean of $\ell_M$ and the score expectation exactly $L_{t,M}$. Dividing instead by the number accepted would define a different random-denominator objective. Resampling until acceptance would change the data law. Neither operation is part of the learner.

The gated conditional regression target need not equal $b_t$: wherever $\E_t[G_M\mid U]>0$, the gated minimizer among arbitrary measurable fields is $\E_t[G_MZ\mid U]/\E_t[G_M\mid U]$. We do not substitute it for the original conditional mean. Instead, we bound gated second moments by ungated velocity risks and use \eqref{eq:proxybias} only for mean-risk transfer. This avoids a hidden assumption that truncation preserves the regression target.

\subsubsection{Tail Formulations and Bounded-Support Recovery}
To relate Assumption~\ref{ass:tails} to directional tails, suppose for a finite $K>0$ that $\Prob(|\langle u,X\rangle|\ge z\mid E)\le2e^{-z^2/K}$ for every unit vector $u$ and $z\ge0$. Each coordinate then satisfies
\begin{equation}
 \E[e^{X_i^2/(3K)}\mid E]
 =1+\frac1{3K}\int_0^\infty e^{s/(3K)}\Prob(X_i^2>s\mid E)\,ds
 \le1+\frac2{3K}\int_0^\infty e^{-2s/(3K)}ds=2.
\end{equation}
Convexity of the exponential gives $\E[e^{\norm X^2/(3Kd)}\mid E]\le d^{-1}\sum_i\E[e^{X_i^2/(3K)}\mid E]\le2$. Thus $\kappa^2=3Kd$ is sufficient without coordinate independence. Conversely, Assumption~\ref{ass:tails} gives $\Prob(\norm X\ge R\mid E)\le2e^{-R^2/\kappa^2}$ and hence directional control. This comparison retains dimension dependence and requires a uniform envelope across rounds.

If endpoints already have known bounded support, choose the proxy to be the original loss, use its positive envelope $B$, and set $\delta=0$. The remaining proofs apply unchanged. If the loss is identically zero, regret vanishes and normalization can be omitted. Under the stated sub-Gaussian model, $c_Z>0$ and the chosen $M$ ensure $B>0$ without such a special case.

\supppage

\section{Useful Results to Prove Theorems 1 and 2}\label{app:common}\label{app:statistics}\label{app:algorithm}
The historical-estimation and aggregation arguments below apply to both theorems. The first-order transfer and lifetime bounds also supply the change-count branch of Theorem~\ref{thm:target} through Theorem~\ref{thm:main}. The final lemma provides the dyadic cover used in both proofs.

\subsection{Historical Statistical Estimation}\label{sec:shared-statistics}
Fix the environment throughout this appendix. Let $H$ be a deterministic interval of $q\ge1$ rounds. It contains $n=mq$ observations, denoted by $\xi_1,\ldots,\xi_n$ in an arbitrary deterministic order preserving each observation's round label. Write $s(i)$ for observation $i$'s round, and $P_i$ for its conditional law. Independence holds across $i$, but the laws $P_i$ may differ. The comparator $v_H^\star$, error $\bar a_H$, and excess $r_H$ are defined in Equation~\eqref{eq:historical}. The population comparator is fixed under the conditioned environment and is not fitted to these observations.

\subsubsection{A Derived Approximate Bernstein Inequality}
\begin{lemma}\label{lem:secondmoment}
For a fixed represented field $v$, let $g_i(v)=\ell_M(v;\xi_i)-\ell_M(v_H^\star;\xi_i)$. Then $|g_i(v)|\le B$ and
\begin{equation}\label{eq:secondmoment}
 \frac1n\sum_i\E g_i(v)^2\le8B\{r_H(v)+2\bar a_H\}.
\end{equation}
Consequently the same right-hand side bounds the average of their variances.
\end{lemma}
For a single target law this is the Bernstein-type variance control of squared-loss flow matching used by \citet{sahoo2026}, in the approximate form that follows from the regression identity. For a heterogeneous history the comparator is the best common field $v_H^\star$, and the residual $\bar a_H$ can be positive even when every round is realizable.
\begin{proof}
Both proxy losses lie in $[0,B]$. On an accepted observation both ungated residuals have norm at most $\sqrt B$, so the difference-of-squares identity yields
\begin{equation}
 g_i(v)^2\le4B G_M(\xi_i)\norm{v(U_i)-v_H^\star(U_i)}^2.
\end{equation}
Insert the ideal field $b_{s(i)}$ appropriate to this observation's law and use $\norm{x+y}^2\le2\norm{x}^2+2\norm{y}^2$:
\begin{equation}
 \E g_i(v)^2\le8B\{Q_{s(i)}(v)+Q_{s(i)}(v_H^\star)\}.
\end{equation}
Dropping the gate is valid because the squared quantities are nonnegative. Each round occurs exactly $m$ times. Averaging over $i$ therefore gives $8B[q^{-1}\sum_sQ_s(v)+\bar a_H]$. The regression identity implies $q^{-1}\sum_sQ_s(v)=r_H(v)+\bar a_H$, since the irreducible regression-noise term cancels separately at each round. Substituting proves \eqref{eq:secondmoment}. Finally $\operatorname{Var}(g_i)\le\E g_i^2$ for each individual law.
\end{proof}
The intermediate function-space inequality is
\begin{equation}\label{eq:fielddistance}
 \frac1q\sum_{s\in H}\E_s\norm{v(U)-v_H^\star(U)}^2
 \le2r_H(v)+4\bar a_H.
\end{equation}
It follows from the same insertion of $b_s$. This is not a parameter-space curvature assumption, and $\bar a_H$ cannot in general be removed. At a tie between two separated class minimizers, excess risk may be zero while their prediction separation is positive. For example, take a scalar base concentrated at zero, a target uniform on $\{-1,1\}$, linear interpolation, and only the two constant fields $v_\pm=\pm c$. Both minimize the loss, yet their sampled loss difference $g=-4cX_1$ has
\begin{equation}\label{eq:misspecified-example}
 L(v_+)-L(v_-)=0,\qquad\operatorname{Var}(g)=16c^2.
\end{equation}
No constant bounds this variance by their zero excess risk, so the approximation term in \eqref{eq:secondmoment} is substantive for nonconvex classes. The example explains a proof boundary and is not a lower bound on achievable regret.

\subsubsection{Uniform Concentration for a Fixed History}
For a fixed $v$, center each observation separately by defining $Y_i(v)=\E g_i(v)-g_i(v)$. The variables are independent and mean zero, with $|Y_i(v)|\le2B$. Their total variance is at most $8Bn[r_H(v)+2\bar a_H]$. The one-sided Bernstein inequality for independent bounded variables \citep{blm2013} then implies, for any $u\ge0$, except on an event of probability at most $e^{-u}$,
\begin{equation}\label{eq:bernstein-fixed}
 \frac1n\sum_i\E g_i(v)-\frac1n\sum_i g_i(v)
 \le4\sqrt{\frac{B[r_H(v)+2\bar a_H]u}{n}}+\frac{2Bu}{n}.
\end{equation}
The standard linear Bernstein remainder $2Bu/(3n)$ is enlarged to $2Bu/n$. No identical-distribution reduction is used.

Choose a deterministic $\zeta$-cover of the proxy-loss class with $K_M$ centers, as in Corollary~\ref{cor:cover}. For a desired $z\ge0$, set $u=\log K_M+z$. Apply \eqref{eq:bernstein-fixed} to each center and sum the failure probabilities: $K_Me^{-u}=e^{-z}$. On the resulting event the inequalities hold at every center simultaneously.

For a general field $v$, let $v_k$ be a center whose proxy loss is within $\zeta$ uniformly. Define $H_u=Bu/n$ for this proof. Empirical comparator differences satisfy
\begin{equation}
 \abs{\frac1n\sum_i g_i(v)-\frac1n\sum_i g_i(v_k)}\le\zeta,
\end{equation}
and the original population excesses satisfy $|r_H(v)-r_H(v_k)|\le\zeta+2\delta$. In addition, the average proxy excess against $v_H^\star$ differs from $r_H(v_k)$ by at most $2\delta$. Combining these three comparisons with \eqref{eq:bernstein-fixed} at $v_k$ gives
\begin{align}
 r_H(v)&\le\frac1n\sum_i g_i(v)+2\zeta+4\delta
 +4\sqrt{H_u[r_H(v_k)+2\bar a_H]}+2H_u\nonumber\\
 &\le\frac1n\sum_i g_i(v)+2\zeta+4\delta
 +4\sqrt{H_u[r_H(v)+2\bar a_H+\zeta+2\delta]}+2H_u.\label{eq:uniform-selfbound}
\end{align}
This holds for every field on one event. A data-dependent approximate minimizer may therefore be substituted without assuming it is independent of the data. The cover centers themselves are fixed before observing the history.

\subsubsection{Proof of Proposition~\ref{thm:historical}}\label{sec:proof-historical}
The original population minimizer $v_H^\star$ is an admissible empirical competitor. By Assumption~\ref{ass:fit},
\begin{equation}
 \frac1n\sum_i g_i(\widehat v)
 =\widehat L_{H,M}(\widehat v)-\widehat L_{H,M}(v_H^\star)
 \le\eopt.
\end{equation}
Let $r=r_H(\widehat v)$ and $x=2\bar a_H+\zeta+2\delta$, both nonnegative. Use $\sqrt{r+x}\le\sqrt r+\sqrt x$ in \eqref{eq:uniform-selfbound}, and then $4\sqrt{H_ur}\le r/2+8H_u$. This yields
\begin{align}
 r&\le\eopt+2\zeta+4\delta+\tfrac12r+10H_u+4\sqrt{H_ux},\nonumber\\
 r&\le2\eopt+4\zeta+8\delta+20H_u+8\sqrt{H_ux}.\label{eq:hp-solved}
\end{align}
Substituting $H_u=B(\log K_M+z)/n$ gives, with probability at least $1-e^{-z}$,
\begin{align}
 r_H(\widehat v)\le{}&2\eopt+4\zeta+8\delta+\frac{20B(\log K_M+z)}n\nonumber\\
 &+8\sqrt{\frac{B(\log K_M+z)}n(2\bar a_H+\zeta+2\delta)}.\label{eq:historicalhp}
\end{align}
This is the high-probability part of Proposition~\ref{thm:historical}. Applying the elementary inequality $8\sqrt{H_ux}\le16H_u+x$ to the final term instead gives
\begin{equation}\label{eq:affinetail}
 r_H(\widehat v)\le2\eopt+36H_u+2\bar a_H+5\zeta+10\delta.
\end{equation}
This inequality also fails with probability at most $e^{-z}$ for every $z\ge0$. Integrating the upper tail is legitimate because the risk is bounded by a finite original-loss envelope. For clarity, if a nonnegative random variable $R$ obeys $\Prob(R>a+cz)\le e^{-z}$ for every $z\ge0$, then
\begin{equation}
 \E R\le a+\int_a^\infty\Prob(R>s)ds
 \le a+c\int_0^\infty e^{-z}dz=a+c.
\end{equation}
Apply this fact to \eqref{eq:affinetail}. It replaces $\log K_M+z$ by $\log K_M+1$, which is at most $\Gamma=1+\log(2K_M)$, and proves \eqref{eq:histfast}.

The square-root interaction can also be retained in expectation. Inequality~\eqref{eq:hp-solved} bounds the upper quantiles by its increasing right side evaluated at a mean-one exponential random variable. Integrate those quantiles and use concavity of the square root. The resulting sharper inequality is
\begin{equation}\label{eq:sharpexpectation}
 \E r_H(\widehat v)\le2\eopt+4\zeta+8\delta
 +\frac{20B\Gamma}{n}
 +8\sqrt{\frac{B\Gamma}{n}(2\bar a_H+\zeta+2\delta)}.
\end{equation}
This auxiliary exponential variable is only a tail-integration device. It is not an additional sampling assumption. In a stationary realizable problem, $\bar a_H=0$ and the leading statistical contribution is $B\Gamma/n$. In a stationary misspecified problem, \eqref{eq:sharpexpectation} retains an interaction of order $\sqrt{B\Gamma\bar a_H/n}$, precluding a blanket $1/n$ proper-selection claim for arbitrary nonconvex classes.

\subsubsection{First-Order Historical-to-Current Transfer}
\begin{lemma}\label{lem:drift}
Let $H=[a,a+q-1]$ precede round $t>a+q-1$. With $\Omega_{a:t}=\sum_{s=a}^{t-1}\omega_s$, the bounds \eqref{eq:histdrift} and \eqref{eq:currentdrift} hold for every represented field used as a historical fit. They hold pointwise, before expectation over its data.
\end{lemma}
\begin{proof}
Choose $v_t^\star\in\argmin_vL_t(v)$. For any $s\in H$, let $v_s^\star$ be a round-$s$ minimizer. The optimality of $v_t^\star$ at $t$ and the telescoping oscillation bound give
\begin{align}
 F_s(v_t^\star)
 &=L_s(v_t^\star)-L_s(v_s^\star)\nonumber\\
 &\le[L_s(v_t^\star)-L_t(v_t^\star)]-[L_s(v_s^\star)-L_t(v_s^\star)]\nonumber\\
 &\le\osc(L_s-L_t)\le2\Omega_{a:t}.
\end{align}
The historical minimizer also minimizes the average of $Q_s$, because the regression-noise offsets are independent of the field. Thus
\begin{equation}
 \bar a_H\le\frac1q\sum_{s\in H}Q_s(v_t^\star)
 =\frac1q\sum_{s\in H}[a_s+F_s(v_t^\star)]
 \le\frac1q\sum_{s\in H}a_s+2\Omega_{a:t}.
\end{equation}
For any field $v$, expand its current excess against $v_t^\star$ around historical average losses:
\begin{align}
 F_t(v)
 &=\bar L_H(v)-\bar L_H(v_t^\star)
 +(L_t-\bar L_H)(v)-(L_t-\bar L_H)(v_t^\star)\nonumber\\
 &\le r_H(v)+\osc(L_t-\bar L_H)
 \le r_H(v)+\frac1q\sum_{s\in H}\osc(L_t-L_s)\nonumber\\
 &\le r_H(v)+2\Omega_{a:t}.
\end{align}
The first inequality uses $\bar L_H(v_H^\star)\le\bar L_H(v_t^\star)$. This proves both claims and displays exactly where common offsets cancel.
\end{proof}
Both bounds are first order in the oscillation budget and hold for every flow-time law. Lemma~\ref{lem:targettransfer} gives second-order counterparts in the target-law budget under Assumption~\ref{ass:smooth}.

\subsubsection{A Localized Rademacher Route}\label{sec:localrad}
\citet{sahoo2026} obtain the fixed-target rate by localized Rademacher complexity \citep{bartlett2005}: they cover the loss class through the parameter-Lipschitz property of the network, bound the local Rademacher complexity by Dudley's entropy integral, and solve a sub-root fixed point. This subsubsection runs the same route for a heterogeneous history and compares the result with Proposition~\ref{thm:historical}. Both routes give a statistical term of order $BP\log(\cdot)/n$, with the same approximation residual.

\paragraph{Setting.} Fix $H=[a,a+q-1]$ and $n=mq$. The examples $\xi_1,\ldots,\xi_n$ are independent, and $\xi_i$ has the law $P_i$ of its round. For a function $g$ of one example let $\bar Pg=n^{-1}\sum_i\E_{P_i}g(\xi_i)$, $P_ng=n^{-1}\sum_ig(\xi_i)$, and $T(g)=n^{-1}\sum_i\E_{P_i}g(\xi_i)^2$. Let $g_v=\ell_M(v;\cdot)-\ell_M(v_H^\star;\cdot)$, $\mathcal G=\{g_v:v\in\V\}$, and $\mathcal G(r)=\{g\in\mathcal G:T(g)\le r\}$. Each $g_v$ takes values in $[-B,B]$, and $g_{v_H^\star}=0$ lies in every $\mathcal G(r)$. For a class $\mathcal F$ write $\mathcal R_n(\mathcal F)=\E\sup_{f\in\mathcal F}n^{-1}\sum_i\sigma_if(\xi_i)$ with independent Rademacher signs $\sigma_i$. Let $L=P\log(1+4R_\Theta L_Mn/B)$.

\paragraph{Step 1: Bernstein condition with a residual.} By \eqref{eq:proxybias}, $r_H(v)\le\bar Pg_v+\delta$. Lemma~\ref{lem:secondmoment} then gives, with $B_0=8B$ and $\rho_0=\delta+2\bar a_H$,
\begin{equation}\label{eq:lr-bernstein}
 T(g_v)\le8B\{r_H(v)+2\bar a_H\}\le B_0(\bar Pg_v+\rho_0)\qquad(v\in\V).
\end{equation}

\paragraph{Step 2: local complexity from the cover.}
\begin{lemma}\label{lem:lr-complexity}
For every $r>0$, $\mathcal R_n(\mathcal G(r))\le24\sqrt{rL/n}+(576L+8)B/n$.
\end{lemma}
\begin{proof}
Subtracting the fixed function $\ell_M(v_H^\star;\cdot)$ preserves supremum-norm distances, so Corollary~\ref{cor:cover} gives a supremum-norm $\varepsilon$-cover of $\mathcal G$ of log size at most $P\log(1+2R_\Theta L_M/\varepsilon)$. A supremum-norm cover is an $L^2(P_n)$ cover, and an external $\varepsilon/2$-cover of the subset $\mathcal G(r)$ yields an internal $\varepsilon$-cover. Hence $\log N(\varepsilon,\mathcal G(r),L^2(P_n))\le P\log(1+4R_\Theta L_M/\varepsilon)$. Conditionally on the sample, Dudley's entropy integral \citep[Lemma~A.3]{srebro2010}, with $\hat\rho^2=\sup_{g\in\mathcal G(r)}P_ng^2$ and the cutoff $\alpha=B/n$, gives
\[
 \E_\sigma\sup_{g\in\mathcal G(r)}\frac1n\sum_i\sigma_ig(\xi_i)\le4\alpha+\frac{12}{\sqrt n}\int_\alpha^{\hat\rho}\sqrt{\log N(\varepsilon,\mathcal G(r),L^2(P_n))}\,d\varepsilon\le\frac{4B}n+12\hat\rho\sqrt{\frac Ln},
\]
because the covering number decreases in $\varepsilon$. If $\hat\rho<\alpha$, Cauchy--Schwarz gives $\E_\sigma\sup_{g\in\mathcal G(r)}n^{-1}\sum_i\sigma_ig(\xi_i)\le\hat\rho<\alpha$, so the same bound holds. Take expectations and apply Jensen's inequality, $\E\hat\rho\le(\E\hat\rho^2)^{1/2}$. Symmetrization for independent observations \citep{blm2013} and the contraction principle \citep{ledoux1991} for $x\mapsto x^2$, which is $2B$-Lipschitz on $[-B,B]$, give
\[
 \E\hat\rho^2\le\sup_{g\in\mathcal G(r)}T(g)+\E\sup_{g\in\mathcal G(r)}(P_n-\bar P)g^2\le r+4B\,\mathcal R_n(\mathcal G(r)).
\]
With $R=\mathcal R_n(\mathcal G(r))$, this yields $R\le4B/n+12\sqrt{Lr/n}+\sqrt{576BLR/n}$. The last term is at most $R/2+288BL/n$, and rearranging proves the lemma.
\end{proof}

\paragraph{Step 3: the fixed point.} The function $\psi(r)=B_0[24\sqrt{rL/n}+(576L+8)B/n]$ is nonnegative, nondecreasing, and $\psi(r)/\sqrt r$ is nonincreasing, so it is sub-root, and $\psi(r)\ge B_0\mathcal R_n(\mathcal G(r))$ by Lemma~\ref{lem:lr-complexity}. Its fixed point solves $r=c\sqrt r+b$ with $c=24B_0\sqrt{L/n}$ and $b=B_0(576L+8)B/n$, so $r^*\le c^2+2b$. With $B_0=8B$,
\begin{equation}\label{eq:lr-fixedpoint}
 \frac{r^*}{B_0}\le\frac{(5760L+16)B}n.
\end{equation}

\paragraph{Step 4: the localized inequality.} Part~1 of Theorem~3.3 of \citet{bartlett2005}, which allows a general functional $T$, is stated for i.i.d.\ samples. Its proof uses the sampling law only through symmetrization and Talagrand's concentration inequality, and both hold for independent, non-identically distributed observations once the law $P$ is replaced by the average $\bar P$ and the variance by $n^{-1}\sum_i\operatorname{Var}_{P_i}$ \citep{kleinrio2005,blm2013}. In this form, for a class $\mathcal F$ with ranges in an interval of length $\ell$ and a functional with $n^{-1}\sum_i\operatorname{Var}_{P_i}f\le T(f)\le B_0\bar Pf$, every $K>1$ and $x>0$ give, with probability at least $1-e^{-x}$ and for all $f\in\mathcal F$,
\begin{equation}\label{eq:lr-bbm}
 \bar Pf\le\frac K{K-1}P_nf+c_1K\frac{r^*}{B_0}+\frac{x(c_2\ell+c_3B_0K)}n,
\end{equation}
where $c_1,c_2,c_3$ are absolute constants and $r^*$ is the fixed point of a sub-root $\psi\ge B_0\mathcal R_n\{f\in\mathcal F:T(f)\le r\}$.

\paragraph{Why the product-sample version holds.} The proof of Part~1 has three ingredients. First, for a star-shaped normalization of $\mathcal F$, it bounds $Z=\sup_f\sum_i\{\E_{P_i}f(\xi_i)-f(\xi_i)\}$ above its mean by a Talagrand-type inequality whose inputs are the range $\ell$ and the variance proxy $\sup_fn^{-1}\sum_i\operatorname{Var}_{P_i}f$. Each summand is centered under its own law $P_i$, and the inequality of \citet{kleinrio2005} gives exactly this bound for independent, non-identically distributed coordinates, with absolute constants. Second, it bounds $\E Z$ by $2n\mathcal R_n$, and symmetrization needs only independence of the coordinates. Third, it combines these bounds over the shells $\{T(f)\le r\}$ through the sub-root property and the fixed point, which is deterministic algebra and does not involve the sampling law. Replacing $P$ by $\bar P$ throughout therefore gives \eqref{eq:lr-bbm}, with constants $c_1,c_2,c_3$ that may differ from those of \citet{bartlett2005}.

Apply \eqref{eq:lr-bbm} to the shifted class $f_v=g_v+\rho_0$ with the functional $\widetilde T(f_v)=T(g_v)$, which is the second moment of $g_v$ rather than of $f_v$. Its range has length $\ell=2B$, its variances equal those of $g_v$ and are at most $\widetilde T(f_v)$, and \eqref{eq:lr-bernstein} reads $\widetilde T(f_v)\le B_0\bar Pf_v$. A constant shift does not change the Rademacher complexity, since $\E\sum_i\sigma_i\rho_0=0$, so the same $\psi$ and $r^*$ apply. Assumption~\ref{ass:fit} gives $P_ng_{\widehat v}\le\eopt$, hence $P_nf_{\widehat v}\le\eopt+\rho_0$. With $K=2$ and $r_H(\widehat v)\le\bar Pg_{\widehat v}+\delta$, with probability at least $1-e^{-x}$,
\begin{equation}\label{eq:lr-final}
 r_H(\widehat v)\le2\eopt+2\bar a_H+2\delta+2c_1\frac{(5760L+16)B}n+\frac{2x(c_2+8c_3)B}n.
\end{equation}
Integrating the tail over $x$ gives the same bound in expectation with $x$ replaced by $1$.

\paragraph{Comparison.} Proposition~\ref{thm:historical} gives $\E r_H(\widehat v)\le36B\Gamma/n+2\bar a_H+2\eopt+5\zeta+10\delta$ with $\Gamma=1+\log(2K_M)$ and, by \eqref{eq:entropy}, $\log K_M\le P\log(1+2R_\Theta L_M(mT)^2/B)$. The localized bound \eqref{eq:lr-final} has statistical term of order $BP\log(1+4R_\Theta L_Mn/B)/n$. The two routes therefore share
\begin{itemize}
\item the factor $B$ from the loss range,
\item the linear factor $P$ from the parametric cover,
\item the depth factor, which enters both logarithms through $C_{R_M}$ in $L_M$, which in general is the product-sum expression Equation~\eqref{eq:paramconstant} with layer-specific bounds $V_j$, of order $D\log(L_\phi V)$ when all layers share one bound $V$,
\item the approximation residual $2\bar a_H$ and the fitting term $2\eopt$.
\end{itemize}
They differ inside the logarithm, where $n$ replaces $(mT)^2$, and in the constants: the cover route has explicit constants, while the localized route inherits the unspecified absolute constants of the concentration inequality. The two routes therefore give the same rate up to logarithmic factors. For a short history, with $n$ much smaller than $mT$, the cover route pays a larger logarithm. This is the same order as the fixed-target bound $\widetilde O(BP/n)$ of \citet{sahoo2026}. Localization therefore does not remove the dependence on $P$ for this class. A width-free bound would need norm-based complexity, which neither analysis assumes. The online analysis uses Proposition~\ref{thm:historical} because its constants are explicit.

\subsection{Aggregation Analysis of DyProd-FM}\label{sec:shared-aggregation}
The algorithm and its initialization rules are given in Appendix~\ref{sec:algorithm-details}. Here we establish the potential, population, and lifetime bounds used by both main theorems.

\subsubsection{Pathwise Potential Bound}
\begin{lemma}\label{lem:potential}
For a lifetime $j=I$, Algorithm~\ref{alg:nsfm} obeys \eqref{eq:master} on every sequence of bounded scores. No stochastic hypothesis is needed.
\end{lemma}
\begin{proof}
The normalized scores lie in $[0,1]$. Their weighted average also lies in $[0,1]$, so $|d_{j,t}|\le1$ and $1+\eta d_{j,t}\ge1/2>0$. The total conceptual weight is initially one. The change in total weight at a round equals
\begin{align}
 \sum_{j\in\mathcal J_t}\eta w_j(\bar x_t-x_{j,t})
 &=\eta\Big[\bar x_t\sum_{j\in\mathcal J_t}w_j
 -\sum_{j\in\mathcal J_t}w_jx_{j,t}\Big]=0,
\end{align}
by the definition of $p_{j,t}$. Sleeping weights do not change, so the total remains one, including weights for future and expired lifetimes. Every individual weight is therefore at most one.

The inequality $\log(1+u)\ge u-u^2$ holds on $[-1/2,1/2]$. Apply it to each update of $j$ and sum over its lifetime:
\begin{equation}
 0\ge\log w_{j,\mathrm{end}}
 \ge-\log M_e+\eta\sum_{t\in I}d_{j,t}-\eta^2\sum_{t\in I}d_{j,t}^2.
\end{equation}
Rearrange and divide by $\eta>0$ to obtain \eqref{eq:master}. Only updates during $I$ enter this product, because the weight is unchanged when the lifetime sleeps.
\end{proof}

\subsubsection{Population Transfer and the Fast Master}
All expectations in this subsubsection condition on the fixed environment. For a lifetime $I$, define $S_I,J_I$ and $\mathcal A_I$ as before \eqref{eq:absorb}.
\begin{lemma}\label{lem:fastmaster}
Every lifetime satisfies
\begin{equation}\label{eq:fastmasterbound}
 S_I\le\frac53J_I+\frac{128}{3}B\log M_e+\frac23\mathcal A_I+\frac83|I|\delta.
\end{equation}
\end{lemma}
\begin{proof}
First condition on $E$ and $\mathcal F_{t-1}$. Write $j=I$ for the fixed lifetime. For the current batch, subtract its score from the mixture score:
\begin{equation}
 d_{j,t}=\sum_kp_{k,t}\frac1{Bm}\sum_{i=1}^m
 [\ell_M(v_{k,t};\xi_{t,i})-\ell_M(v_{j,t};\xi_{t,i})].
\end{equation}
Jensen's inequality over the probabilities and over the $m$ observations gives
\begin{equation}\label{eq:batch-jensen}
 \E[d_{j,t}^2\mid E,\mathcal F_{t-1}]
 \le\frac1{B^2}\sum_kp_{k,t}\E_t[(\ell_M(v_{k,t})-\ell_M(v_{j,t}))^2].
\end{equation}
The difference-of-squares calculation in Lemma~\ref{lem:secondmoment} gives, for any two represented fields $u,v$ fixed under this conditioning,
\begin{equation}
 \E_t[(\ell_M(u)-\ell_M(v))^2]
 \le4B\E_t\norm{u(U)-v(U)}^2
 \le8B[Q_t(u)+Q_t(v)].
\end{equation}
Substitution into \eqref{eq:batch-jensen} proves \eqref{eq:mastervariance}. It bounds a second moment, not just the batch-noise variance: the squared population score difference remains even for very large $m$.

The conditional mean score difference is exactly the difference of population proxy risks, divided by $B$. Proper selection in Equation~\eqref{eq:conditional-selection} and the uniform original-risk transfer give
\begin{equation}\label{eq:score-pop-transfer}
 S_I-J_I\le B\sum_{t\in I}\E d_{j,t}+2|I|\delta.
\end{equation}
The subtraction of $a_t$ cancels between the two compared risks. For $\eta=1/32$, take expectations in the potential bound, then substitute the second-moment estimate. Since $B\eta\cdot8/B=1/4$ and $Q_t=F_t+a_t$,
\begin{equation}
 S_I-J_I\le32B\log M_e+\tfrac14(S_I+J_I+2\mathcal A_I)+2|I|\delta.
\end{equation}
Subtract $S_I/4$ and move $J_I$ to the right. The left becomes $3S_I/4$, and the comparator coefficient becomes $5J_I/4$. Division by $3/4$ proves \eqref{eq:fastmasterbound}.
\end{proof}
Only the potential inequality is pathwise. The population statement additionally uses fresh regression feedback, predictable candidates, and squared loss. The two statements should not be conflated. The second moment in \eqref{eq:mastervariance} does not decay with the batch size: for the per-example normalized score difference $\Delta$ between the mixture and lifetime $j$, with independent batch copies $\Delta_1,\ldots,\Delta_m$,
\begin{equation}\label{eq:batch-second-moment}
 \E_{\mathrm{new}}\Big(\frac1m\sum_i\Delta_i\Big)^2=(\E_{\mathrm{new}}\Delta)^2+\frac{\operatorname{Var}_{\mathrm{new}}(\Delta)}m.
\end{equation}
The first term survives as $m$ grows, which is why the constants separate the $1/m$ gain of base estimation from the $m$-free aggregation cost. A square-root master rate $\eta\asymp|I|^{-1/2}$ is not needed: every bound in this paper is obtained from \eqref{eq:fastmasterbound}.

\subsubsection{Historical Fits over a Lifetime}
Define $e_M=2\eopt+5\zeta+10\delta$ for this subsubsection. For any lifetime $I$, let $n_I=|I|$, and let $\Omega_I$ sum only its internal variation increments.
\begin{lemma}\label{lem:base}
The fitted model of every lifetime satisfies
\begin{equation}\label{eq:basefast}
 J_I\le B+\frac{72B\Gamma}{m}\log(en_I)+4\log(en_I)\mathcal A_I
 +n_Ie_M+6(n_I-1)\Omega_I.
\end{equation}
\end{lemma}
\begin{proof}
Write the lifetime's start as $a$. At its first round the deployed field $v_{I,a}$ is a predictable member of $\V$. Nonnegativity of the minimum loss and \eqref{eq:proxybias} give $F_a(v_{I,a})\le L_a(v_{I,a})\le B+\delta$.

At a later round let $k=t-a\ge1$ be the age and $q=2^{\lfloor\log_2k\rfloor}$ its historical prefix length. Then $k/2\le q\le k$. Apply \eqref{eq:histfast} to this deterministic prefix and then both inequalities of Lemma~\ref{lem:drift}. The resulting current excess obeys
\begin{equation}\label{eq:agefast}
 \E F_t(v_{I,t})\le\frac{36B\Gamma}{mq}
 +\frac2q\sum_{s=a}^{a+q-1}a_s+e_M+6\Omega_{a:t}.
\end{equation}
The drift coefficient is six: the factor two multiplying $\bar a_H$ in the statistical bound contributes four through \eqref{eq:histdrift}, and current-risk transfer contributes two.

For the statistical term, $1/q\le2/k$ and the harmonic sum gives
\begin{equation}
 \sum_{k=1}^{n_I-1}\frac1q\le2\sum_{k=1}^{n_I-1}\frac1k\le2\log(en_I).
\end{equation}
For a fixed earlier round $s\in I$, its coefficient in the double approximation sum is at most $\sum_{k=1}^{n_I-1}2/q\le4\log(en_I)$. Summing this bound over $s$ gives the approximation term in \eqref{eq:basefast}. Every internal drift edge can occur in at most $n_I-1$ of the quantities $\Omega_{a:t}$. Thus $\sum_{t>a}\Omega_{a:t}\le(n_I-1)\Omega_I$. There are $n_I-1$ fitting residuals, and the unused one-round allowance $e_M\ge\delta$ absorbs the initial tail term. This proves \eqref{eq:basefast}, including $n_I=1$.
\end{proof}
No union bound over experts is needed: each expectation inequality holds for its deterministic prefix. Linearity of expectation permits summing these bounds even when their training data overlap. Lemma~\ref{lem:basetarget} is the target-law counterpart of this lemma and uses the same prefixes and the same fits.

\subsubsection{Resource Accounting and Computational Scope}\label{sec:resources}
Let $N=2^{\lceil\log_2T\rceil}$ denote the containing power of two. There are $2N-1<4T$ binary-tree nodes, so $M_e<4T$. At each scale, at most one original node contains a given round, and clipping and removing duplicates cannot increase this count. Hence at most $1+\lceil\log_2T\rceil$ lifetime models are active. Storing their parameters costs $O(P\log T)$ scalar entries. A direct implementation stores at most $mT$ triples, or $O(mdT)$ scalars, in addition to $O(T)$ conceptual weights. Future weights may be initialized lazily at their preassigned value, and expired weights never need to be updated again.

An original interval of length $2^k$ triggers at most $k$ nontrivial fitting calls, because prefix lengths are $1,2,\ldots,2^{k-1}$. There are $N/2^k$ original nodes of that length. Therefore the total number of fitting calls is at most $N\sum_{k\ge1}k/2^k\le2N=O(T)$, and clipping and deduplication do not increase it. The sum of prefix lengths within one original node is less than its length, and total original-node lengths per scale equal $N$. Thus the total size of datasets passed to fits, counting repeated historical examples, is $O(mT\log T)$. This is a data-volume count, not an optimizer running time: a fitting routine may revisit its dataset many times. The load is uneven in time. A round at which a long prefix doubles retrains on up to $O(mT)$ examples, while the amortized load is $O(m\log T)$ examples per round.

Scoring uses $O(mT\log T)$ field evaluations. If a field evaluation costs $C_{\mathrm{eval}}$, the corresponding cost is $O(C_{\mathrm{eval}}mT\log T)$, and master arithmetic is $O(T\log T)$. No dimension or minibatch factor is silently removed from these counts. The general computational guarantee remains conditional on the cost of a fitting routine satisfying Assumption~\ref{ass:fit}.

\subsection{Disjoint Dyadic Covers}
\begin{lemma}[Disjoint dyadic covers]\label{lem:dyadiccover}
Every consecutive interval $J\subseteq[1,T]$ admits a partition into at most $D_T=2(1+\lceil\log_2T\rceil)$ disjoint lifetimes in $\mathcal D_T$, with at most two pieces of each length.
\end{lemma}
\begin{proof}
Fix any consecutive interval $J\subseteq[1,T]$. Use the untruncated aligned dyadic intervals that are maximal among those contained in $J$. They partition $J$: each point belongs to a singleton and hence to a largest admissible containing dyadic interval, and two maximal dyadic intervals are either disjoint or equal. There are at most two such intervals of any length. Indeed, the parent of a maximal interval is not contained in $J$, so it crosses a boundary of $J$. At a fixed scale at most one parent crosses each boundary, and such a parent has at most one child contained in $J$. Each chosen interval is an original tree node contained in $[1,T]$, so it is among the algorithm's lifetimes. The number of pieces is therefore at most $D_T=2(1+\lceil\log_2T\rceil)$.
\end{proof}

\supppage

\section{Proof of Theorem 1}\label{app:regret}\label{sec:proof-theorem1}
We prove Theorem~\ref{thm:main} without an asymptotic optimization over real-valued block lengths. First fix the environment. All expectations until the final step are conditional on it, and $\mathcal A_J$ is consequently deterministic.

\medskip\noindent\textbf{Step 1: Combine the Master and Base Bounds.}
For a dyadic lifetime $I$, combine the historical lifetime bound of Lemma~\ref{lem:base}, Equation~\eqref{eq:basefast}, with the fast master bound of Lemma~\ref{lem:fastmaster}, Equation~\eqref{eq:fastmasterbound}, both proved in Appendix~\ref{app:common}. Lemma~\ref{lem:base} uses Proposition~\ref{thm:historical} and the first-order transfer of Lemma~\ref{lem:drift}. Lemma~\ref{lem:fastmaster} combines the potential inequality of Lemma~\ref{lem:potential} with the second-moment control of Lemma~\ref{lem:secondmoment}. The non-statistical constant is $(5/3)B+(128/3)B\log M_e$, which is at most $44B\Lambda$. The statistical coefficient becomes $(5/3)72=120$. The approximation coefficient is
\begin{equation}
 \frac{20}{3}\log(e|I|)+\frac23\le8\log(eT)=H_T.
\end{equation}
The coefficients of $\eopt,\zeta,\delta$ are $10/3$, $25/3$, and $50/3+8/3=58/3$, respectively. They are at most $4,9,20$. Finally $(5/3)6=10$ for drift. Thus
\begin{equation}\label{eq:dyadicfast}
 S_I\le K_T/D_T+H_T\mathcal A_I+|I|e_f+10(|I|-1)\Omega_I.
\end{equation}
The same fields and selected sequence enter every later bound, including the target-law bound of Appendix~\ref{app:targetproof}. There is no comparison between separate learner runs.

\medskip\noindent\textbf{Step 2: Apply the Disjoint Dyadic Cover.}
Apply Lemma~\ref{lem:dyadiccover} to any consecutive interval $J\subseteq[1,T]$.

Approximation and residual sums over the cover have no overlap. Internal variation edges of its pieces are a subset of those of $J$, and $|I|-1\le|J|-1$. Consequently
\begin{equation}
 \sum_{I\ \mathrm{in\ cover}}(|I|-1)\Omega_I\le(|J|-1)\Omega_J.
\end{equation}
Summing \eqref{eq:dyadicfast} over the cover gives exactly \eqref{eq:intervalfast}.

\medskip\noindent\textbf{Step 3: Change-Count Partition.}
Partition $[1,T]$ immediately after every edge with $\omega_t>0$. Let $k$ be the resulting number of intervals, and let $n_1,\ldots,n_k$ be their lengths. Then $k\le N_\omega+1$, $\sum_jn_j=T$, and each interval has zero internal variation. Applying the interval bound established in Step~2, Equation~\eqref{eq:intervalfast}, to each interval gives
\begin{equation}\label{eq:countfastproof}
 \E[\Reg\mid E]\le K_Tk+H_T\App+Te_f
 \le K_T(N_\omega+1)+H_T\App+Te_f.
\end{equation}
The partition depends only on the conditioned environment, and the algorithm never needs to observe it.

\medskip\noindent\textbf{Step 4: Variation Partition.}
Choose an integer $h\in[1,T]$ for the proof and partition $[1,T]$ into $\lceil T/h\rceil$ consecutive blocks of length at most $h$. Each block has at most $h-1$ internal edges, and distinct blocks have disjoint internal edges, so summing the interval bound from Step~2, Equation~\eqref{eq:intervalfast}, gives, after removing the global approximation and residual terms, the integer optimization
\begin{equation}\label{eq:integer-fast}
 \min_{h\in\{1,\ldots,T\}}
 \{K_T\lceil T/h\rceil+10(h-1)\Omega\}.
\end{equation}
If $\Omega=0$, choose $h=T$ and obtain $K_T$. If $\Omega>0$, define $h_* =\sqrt{K_TT/(10\Omega)}$. There are three cases.
\begin{enumerate}
\item If $1\le h_*\le T$, set $h=\lceil h_*\rceil$. Use $\lceil T/h\rceil\le T/h+1$, $h\ge h_*$, and $h-1\le h_*$ to get
\begin{equation}
 K_T\lceil T/h\rceil+10(h-1)\Omega
 \le K_T+K_TT/h_*+10h_*\Omega
 =K_T+2\sqrt{10K_TT\Omega}.
\end{equation}
\item If $h_*>T$, choose $h=T$. Let $x=10T\Omega$. The condition implies $x<K_T$, since $10\Omega<K_T/T$. Therefore $10(T-1)\Omega\le x\le\sqrt{K_Tx}$, and the cost is at most $K_T+\sqrt{10K_TT\Omega}$.
\item If $h_*<1$, choose $h=1$. Now $K_TT<10\Omega$, which gives $K_TT\le\sqrt{10K_TT\Omega}$. The cost is $K_TT$, again bounded by the desired expression.
\end{enumerate}
Restoring $H_T\App+Te_f$ and comparing with \eqref{eq:countfastproof} proves Equation~\eqref{eq:fast} conditionally on $E$. The proof handles $T=1$ and the transitions between regimes without assuming a real-valued window can be used by the algorithm.

\medskip\noindent\textbf{Step 5: Integrate the Environment.}
The separate count and variation bounds hold for every admissible fixed environment. The deterministic caps and constants are uniform, so integration preserves them and replaces $\App$ with $\E\App$. Both unconditional bounds therefore hold, and their minimum also bounds $\E\Reg$. This is not an interchange of expectation and a random minimum. For completeness, the separate uniform bound follows from $F_t(v)\le L_t(v)\le B+\delta$ for every represented field, giving $\E\Reg\le T(B+\delta)$. The stated theorem keeps the count and variation bounds derived above.

At fixed architectural and endpoint scales, Equation~\eqref{eq:gatechoice} gives $B=O(\log(emT))$, $R_M=O(\sqrt{\log(emT)})$, and $C_{R_M}=O(1+R_M)$. Inequality~\eqref{eq:entropy} with $\varepsilon=\zeta$ then gives $\Gamma=O(1+P\log(emT))$, with fixed constants retained. Also $M_e<4T$, $D_T=O(\log(eT))$, and $T\zeta=B/(m^2T)$. These relations justify the suppressed logarithms in the abstract. If architecture, batch size, or endpoint scale varies with the horizon, the explicit theorem is the relevant statement.

Thus, under the predictive protocol, Assumptions~\ref{ass:tails}--\ref{ass:fit}, and the caps in Equation~\eqref{eq:caps}, Algorithm~\ref{alg:nsfm} satisfies
\begin{align}
 \E\Reg\le{}&\Fast,\nonumber\\
 \Fast={}&\min\big\{K_T(N_\omega+1),\ K_T+2\sqrt{10K_TT\Omega}\big\}\nonumber\\
 &+H_T\E\App+Te_f.\label{eq:proof-main-conclusion}
\end{align}
Hence Theorem~\ref{thm:main} is proved. \hfill$\square$

\supppage

\section{Proof of Theorem 2}\label{app:smoothing}\label{app:targetproof}
\subsection{Supporting Lemmas Needed for Theorem 2}\label{sec:support-theorem2}
This appendix proves Theorem~\ref{thm:target}. We first establish Lemma~\ref{lem:smoothing} and Propositions~\ref{prop:link} and~\ref{prop:necessity}, and compute the flow-time moments of Assumption~\ref{ass:smooth}. We then derive the historical and interval bounds and complete the proof using the total-variation partition. Throughout, Assumptions~\ref{ass:tails},~\ref{ass:network} and~\ref{ass:smooth} hold, and the environment is fixed until the final integration step. All results of Appendices~\ref{app:consequences}--\ref{app:algorithm} used here hold under Assumption~\ref{ass:smooth}'s flow-time law by Remark~\ref{rem:timelaw}. For laws $P,Q$ with densities $p,q$ with respect to a common measure, the Hellinger distance is
\begin{equation}
 \Hel^2(P,Q)=\frac12\int(\sqrt p-\sqrt q)^2=1-\int\sqrt{pq}\in[0,1].
\end{equation}

\subsubsection{Hellinger Tools}
\begin{lemma}[Hellinger change of measure]\label{lem:hellinger-cm}
For every measurable $F$ with $\int F^2d(P+Q)<\infty$,
\begin{equation}
 \Big|\int F\,dP-\int F\,dQ\Big|\le2\Hel(P,Q)\Big(\int F^2dP+\int F^2dQ\Big)^{1/2}.
\end{equation}
\end{lemma}
\begin{proof}
Write $\int F(p-q)=\int F(\sqrt p+\sqrt q)(\sqrt p-\sqrt q)$ and apply Cauchy--Schwarz:
\[
 \Big|\int F(p-q)\Big|\le\Big(\int F^2(\sqrt p+\sqrt q)^2\Big)^{1/2}\Big(\int(\sqrt p-\sqrt q)^2\Big)^{1/2}.
\]
The first factor is at most $(2\int F^2(p+q))^{1/2}$ because $(\sqrt p+\sqrt q)^2\le2(p+q)$, and the second equals $\sqrt2\,\Hel(P,Q)$.
\end{proof}

\begin{lemma}[Hellinger cost of translating the base]\label{lem:hellinger-loc}
For $\sigma\in\R^d$, let $\nu_\sigma$ be the law of $X_0+\sigma$ with $X_0\sim\nu$. Then $\Hel^2(\nu_\sigma,\nu)\le\iota\norm\sigma^2/8$.
\end{lemma}
\begin{proof}
Let $f=\sqrt{p_\nu}\in H^1(\R^d)$. For smooth compactly supported $f$, the fundamental theorem of calculus and Jensen give
\[
 |f(x-\sigma)-f(x)|^2=\Big|\int_0^1\langle\nabla f(x-\theta\sigma),\sigma\rangle d\theta\Big|^2
 \le\int_0^1\langle\nabla f(x-\theta\sigma),\sigma\rangle^2d\theta.
\]
Integrating in $x$ and using translation invariance of Lebesgue measure gives $\int|f(\cdot-\sigma)-f|^2\le\int\langle\nabla f,\sigma\rangle^2$. Both sides are continuous in $H^1$, so the inequality extends to every $f\in H^1(\R^d)$ by density. Since $I_\nu=4\int\nabla f\nabla f^\top$, the right side equals $\sigma^\top I_\nu\sigma/4\le\iota\norm\sigma^2/4$. Divide by two.
\end{proof}
For $\nu=N(0,\sigma_0^2I_d)$ the exact value is $\Hel^2(\nu_\sigma,\nu)=1-\exp(-\norm\sigma^2/(8\sigma_0^2))$, which the lemma bounds by $\norm\sigma^2/(8\sigma_0^2)$ because $1-e^{-x}\le x$.

\subsubsection{The Base-Shift Representation}
\begin{lemma}[A target move is a base-noise shift plus a label shift]\label{lem:baseshift}
Let $(X_1,X_1')$ have marginals $(\mu_s,\mu_t)$, and let $X_0\sim\nu$ and $r\sim\lambda$ be independent of it and of each other. On $\{\alpha(r)>0\}$ put $\Delta X=X_1'-X_1$, $\sigma_r=(\beta(r)/\alpha(r))\Delta X$, $e_r=(J(r)/\alpha(r))\Delta X$, and $\tilde X_0=X_0+\sigma_r$. Then Equation~\eqref{eq:baseshift} holds, and conditional on $(X_1,X_1',r)$ the law of $\tilde X_0$ is $\nu_{\sigma_r}$.
\end{lemma}
\begin{proof}
First, $\alpha\tilde X_0+\beta X_1=\alpha X_0+\beta\Delta X+\beta X_1=\alpha X_0+\beta X_1'$. Second, using $J/\alpha=\beta'-\alpha'\beta/\alpha$,
\[
 \alpha'\tilde X_0+\beta'X_1+\frac J\alpha\Delta X
 =\alpha'X_0+\frac{\alpha'\beta}{\alpha}\Delta X+\beta'X_1+\beta'\Delta X-\frac{\alpha'\beta}{\alpha}\Delta X
 =\alpha'X_0+\beta'X_1'.
\]
The conditional law statement holds because $X_0$ is independent of $(X_1,X_1',r)$ and $\sigma_r$ is a function of these variables.
\end{proof}
The pair on the left of Equation~\eqref{eq:baseshift}, together with $r$, has the round-$t$ example law, while $(I_r(X_0,X_1),Z_r(X_0,X_1),r)$ has the round-$s$ law.

\subsubsection{Label Tails}
\begin{lemma}[Label tails]\label{lem:responsetail}
For every round $t$ and $M\ge0$, with $M_z=c_Z(1+M)$,
\begin{equation}
 \E_t\big[\norm Z^2\ind\{\norm Z>M_z\}\big]\le4c_Z^2(1+M^2+4\kappa^2)e^{-M^2/(4\kappa^2)}\le\delta_M/2.
\end{equation}
\end{lemma}
\begin{proof}
Let $S=\norm{X_0}+\norm{X_1}$. By Assumption~\ref{ass:smooth}(a), $\norm Z\le c_Z S\le c_Z(1+S)$, so $\{\norm Z>M_z\}\subseteq\{S>M\}$ and $\norm Z^2\le2c_Z^2(1+S^2)$. The tail computation in the proof of Proposition~\ref{prop:tail} gives $\E_t[(1+S^2)\ind\{S^2>M^2\}]\le2(1+M^2+4\kappa^2)e^{-M^2/(4\kappa^2)}$. Multiply by $2c_Z^2$. Finally $C_\ell\ge4c_Z^2$, so the bound is at most $(2c_Z^2/C_\ell)\delta_M\le\delta_M/2$.
\end{proof}

\subsubsection{Proof of Lemma~\ref{lem:smoothing}}\label{sec:proof-smoothing}
Fix rounds $s,t$ and $u,w\in\V$, write $\Delta=u-w$, and let $W=W_2(\mu_s,\mu_t)$, which is finite by Assumption~\ref{ass:tails}. If $W=0$ the two rounds have the same example law and both sides vanish. Otherwise take a $W_2$-optimal coupling $(X_1,X_1')$, which exists on $\R^d$, and the variables of Lemma~\ref{lem:baseshift}. Write $U=(I_r(X_0,X_1),r)$, $Z=Z_r(X_0,X_1)$ for the round-$s$ example and $U',Z'$ for the round-$t$ example built from $X_1'$. Define
\[
 \varphi(U,Z)=\norm{u(U)-Z}^2-\norm{w(U)-Z}^2=\langle\Delta(U),u(U)+w(U)-2Z\rangle.
\]
It is integrable because $|\varphi|\le2B_v(2B_v+2\norm Z)$. Then $D_{s,t}(u)-D_{s,t}(w)=\E\varphi(U',Z')-\E\varphi(U,Z)$.

\textbf{Step 1: separate the two shifts.} By Equation~\eqref{eq:baseshift}, $U'=(I_r(\tilde X_0,X_1),r)$ and $Z'=Z_r(\tilde X_0,X_1)+e_r$. Expanding the label shift,
\[
 \varphi(U',Z')=\varphi\big(U',Z_r(\tilde X_0,X_1)\big)-2\langle\Delta(U'),e_r\rangle.
\]
Let $F(x_0)=\varphi((I_r(x_0,X_1),r),Z_r(x_0,X_1))$, a function of $x_0$ given $(X_1,r)$. Then $\varphi(U',Z_r(\tilde X_0,X_1))=F(\tilde X_0)$ and $\varphi(U,Z)=F(X_0)$, so
\begin{equation}\label{eq:twoterms}
 D_{s,t}(u)-D_{s,t}(w)=\E[F(\tilde X_0)-F(X_0)]-2\E\langle\Delta(U'),e_r\rangle.
\end{equation}

\textbf{Step 2: the label shift.} Since $\norm{e_r}=k(r)\norm{\Delta X}$ and $r$ is independent of $(X_1,X_1')$, Cauchy--Schwarz gives
\[
 2|\E\langle\Delta(U'),e_r\rangle|\le2(\E\norm{\Delta(U')}^2)^{1/2}(\E k(r)^2\norm{\Delta X}^2)^{1/2}=2\sqrt{m_k}\,W\norm\Delta_t.
\]

\textbf{Step 3: the base-noise shift.} Condition on $\mathcal G=\sigma(X_1,X_1',r)$. By Lemma~\ref{lem:baseshift}, $X_0\sim\nu$ and $\tilde X_0\sim\nu_{\sigma_r}$ conditionally. Lemmas~\ref{lem:hellinger-cm} and~\ref{lem:hellinger-loc} give
\begin{equation}\label{eq:inputstep}
 |\E[F(\tilde X_0)-F(X_0)\mid\mathcal G]|\le2h(r)\norm{\Delta X}\big(\E[F(X_0)^2\mid\mathcal G]+\E[F(\tilde X_0)^2\mid\mathcal G]\big)^{1/2}.
\end{equation}
To bound the second moments, split the label at $M_z$. Since $\norm\Delta\le2B_v$ and $\norm{u+w}\le2B_v$,
\begin{align*}
 |F(X_0)|&\le(2B_v+2M_z)\norm{\Delta(U)}+4B_v\norm Z\ind\{\norm Z>M_z\},\\
 |F(\tilde X_0)|&=|\varphi(U',Z'-e_r)|\\
 &\le(2B_v+2M_z)\norm{\Delta(U')}+4B_v\norm{Z'}\ind\{\norm{Z'}>M_z\}+4B_v\norm{e_r}.
\end{align*}
Apply $(a+b+c)^2\le3(a^2+b^2+c^2)$, then $\sqrt{\sum_ix_i}\le\sum_i\sqrt{x_i}$. With $A_s=\E[\norm{\Delta(U)}^2\mid\mathcal G]$, $A_t=\E[\norm{\Delta(U')}^2\mid\mathcal G]$, and $\tau_s,\tau_t$ the conditional label-tail moments of $Z$ and $Z'$, the bracket in \eqref{eq:inputstep} has square root at most
\[
 \sqrt3(2B_v+2M_z)(\sqrt{A_s}+\sqrt{A_t})+4\sqrt3B_v(\sqrt{\tau_s}+\sqrt{\tau_t})+4\sqrt3B_vk(r)\norm{\Delta X}.
\]
Multiply by $2h(r)\norm{\Delta X}$ and take expectations. By Cauchy--Schwarz and independence of $r$ and $(X_1,X_1')$, $\E[h(r)\norm{\Delta X}\sqrt{A_s}]\le(\E h^2\,W^2)^{1/2}\norm\Delta_s$, and similarly with $A_t$ and $\norm\Delta_t$. Since $Z$ and $Z'$ have the round-$s$ and round-$t$ label laws, Lemma~\ref{lem:responsetail} gives $\E\tau_s,\E\tau_t\le\delta_M/2$. Finally $\E[h(r)k(r)\norm{\Delta X}^2]=\E_\lambda[hk]W^2$. With $\E_\lambda h^2=\iota m_h/8$,
\begin{align*}
 |\E[F(\tilde X_0)-F(X_0)]|\le{}&4\sqrt3(B_v+M_z)\sqrt{\iota m_h/8}\,W(\norm\Delta_s+\norm\Delta_t)\\
 &+16\sqrt3B_v\sqrt{\iota m_h/8}\sqrt{\delta_M/2}\,W+8\sqrt3B_v\E_\lambda[hk]W^2.
\end{align*}

\textbf{Step 4: combine.} Adding Step~2 to this bound and using the definitions in Equation~\eqref{eq:smoothing-constants} proves the lemma. If endpoints are bounded and $\norm Z\le M_z$ almost surely, the tail moments vanish and $\esm$ can be replaced by zero. \hfill$\square$

The proof never differences a represented field at two inputs, so the smoothing transfer needs no additional Lipschitz or norm-equivalence assumption on the class, and only the output bound $B_v$ enters. It also never uses the ideal field or regularity of the targets.

\subsubsection{Approximate Quadratic Growth}
\begin{lemma}\label{lem:aqg}
For every round $t$, every $u\in\V$, and every minimizer $v_t^\star$ of $L_t$ over $\V$, $\norm{u-v_t^\star}_t\le\sqrt{F_t(u)}+2\sqrt{a_t}$.
\end{lemma}
\begin{proof}
The triangle inequality through the ideal field gives $\norm{u-v_t^\star}_t\le\sqrt{Q_t(u)}+\sqrt{Q_t(v_t^\star)}$. By Equation~\eqref{eq:excess}, $Q_t(u)=F_t(u)+a_t$ and $Q_t(v_t^\star)=a_t$, so the right side is at most $\sqrt{F_t(u)}+\sqrt{a_t}+\sqrt{a_t}$.
\end{proof}
This is the flow-matching substitute for quadratic growth: it holds for every nonconvex class, with the approximation error as its only residual. It is exact in the realizable case.

\subsubsection{Second-Order Drift Transfer}
\begin{lemma}\label{lem:secondorder}
Let $s,t$ be rounds and suppose that for some $\hbar,\rho\ge0$ and all $u,w\in\V$
\begin{equation}\label{eq:hyp-second}
 |D_{s,t}(u)-D_{s,t}(w)|\le\hbar(\norm{u-w}_s+\norm{u-w}_t)+\rho.
\end{equation}
Then, with $v_s^\star,v_t^\star$ minimizers of $L_s,L_t$ over $\V$:
\begin{enumerate}
\item[(i)] $F_t(v_s^\star)\le5\hbar^2+8(a_s+a_t)+2\rho$ and $F_s(v_t^\star)\le5\hbar^2+8(a_s+a_t)+2\rho$,
\item[(ii)] $F_t(u)\le3F_s(u)+14\hbar^2+12a_s+10a_t+4\rho$ for every $u\in\V$.
\end{enumerate}
\end{lemma}
\begin{proof}
\emph{Sandwich.} Since $\min_\V L_s\le L_s(v_t^\star)$, for every $u$,
\begin{equation}\label{eq:sandwich}
 F_t(u)-F_s(u)=D_{s,t}(u)-\big(\min_\V L_t-\min_\V L_s\big)\le D_{s,t}(u)-D_{s,t}(v_t^\star).
\end{equation}
Exchanging the roles of $s$ and $t$, and using $D_{t,s}=-D_{s,t}$, gives $F_s(u)-F_t(u)\le D_{s,t}(v_s^\star)-D_{s,t}(u)$.

\emph{Part (i).} Let $x=\sqrt{F_s(v_t^\star)}$, $y=\sqrt{F_t(v_s^\star)}$, $z=\max\{x,y\}$ and $A=2\sqrt{a_s}+2\sqrt{a_t}$. Apply \eqref{eq:sandwich} with $u=v_s^\star$ and its exchanged form with $u=v_t^\star$. Both give an upper bound $D_{s,t}(v_s^\star)-D_{s,t}(v_t^\star)$, which by \eqref{eq:hyp-second} is at most $\hbar(\norm{v_s^\star-v_t^\star}_s+\norm{v_s^\star-v_t^\star}_t)+\rho$. Lemma~\ref{lem:aqg} at round $s$ with $u=v_t^\star$ and at round $t$ with $u=v_s^\star$ bounds the two norms by $x+2\sqrt{a_s}$ and $y+2\sqrt{a_t}$. Hence $z^2\le\hbar(2z+A)+\rho$, so $z\le\hbar+\sqrt{\hbar^2+\hbar A+\rho}$ and
\[
 z^2\le2\hbar^2+2(\hbar^2+\hbar A+\rho)\le5\hbar^2+A^2+2\rho\le5\hbar^2+8(a_s+a_t)+2\rho,
\]
using $2\hbar A\le\hbar^2+A^2$ and $A^2\le8(a_s+a_t)$.

\emph{Part (ii).} By \eqref{eq:sandwich} and \eqref{eq:hyp-second}, $F_t(u)-F_s(u)\le\hbar(\norm{u-v_t^\star}_s+\norm{u-v_t^\star}_t)+\rho$. Lemma~\ref{lem:aqg} gives $\norm{u-v_t^\star}_t\le\sqrt{F_t(u)}+2\sqrt{a_t}$ and
\[
 \norm{u-v_t^\star}_s\le\norm{u-v_s^\star}_s+\norm{v_s^\star-v_t^\star}_s\le\sqrt{F_s(u)}+z+4\sqrt{a_s}.
\]
Apply $\hbar\sqrt{F_t}\le F_t/2+\hbar^2/2$, $\hbar\sqrt{F_s}\le F_s/2+\hbar^2/2$, $\hbar z\le(\hbar^2+z^2)/2$, $2\hbar\sqrt{a_t}\le\hbar^2+a_t$, and $4\hbar\sqrt{a_s}\le2\hbar^2+2a_s$. Collecting terms,
\[
 \tfrac12F_t(u)\le\tfrac32F_s(u)+\tfrac92\hbar^2+\tfrac12z^2+a_t+2a_s+\rho.
\]
Multiply by two and insert part (i) for $z^2$.
\end{proof}

\subsubsection{Proof of Proposition~\ref{prop:link}}\label{sec:proof-link}
By definition $2\omega_{s,t}=\sup_uD_{s,t}(u)-\inf_wD_{s,t}(w)=\sup_{u,w}(D_{s,t}(u)-D_{s,t}(w))$. Lemma~\ref{lem:smoothing} and $\norm{u-w}_r\le2B_v$ give $2\omega_{s,t}\le4B_v\Lsm W+\esm W+c_2W^2$, which is \eqref{eq:link-first}. For consecutive rounds, summation gives $\Omega_T\le2B_v\Lsm V_T^\mu+\frac12\esm V_T^\mu+\frac12c_2\sum_tW_2(\mu_{t+1},\mu_t)^2$. Inequality \eqref{eq:link-second} is part~(i) of Lemma~\ref{lem:secondorder} with $\hbar=\Lsm W$ and $\rho=\esm W+c_2W^2$, whose hypothesis is Lemma~\ref{lem:smoothing}. \hfill$\square$

\subsubsection{Proof of Proposition~\ref{prop:necessity}}\label{sec:proof-necessity}
Given $x_1$, $I_r(X_0,x_1)=\alpha X_0+\beta x_1\sim N(\beta x_1,\alpha^2\sigma_0^2I_d)$. For two such Gaussians with common covariance the Bhattacharyya coefficient is $\exp(-\beta^2\norm{x_1-x_1'}^2/(8\alpha^2\sigma_0^2))$, so
\[
 \Hel^2(q_{r,x_1},q_{r,x_1'})=1-\exp\Big(-\frac{\beta^2\norm{x_1-x_1'}^2}{8\alpha^2\sigma_0^2}\Big).
\]
Since $1-e^{-x}\le x$ with ratio tending to one as $x\downarrow0$, the supremum of $\Hel/\norm{x_1-x_1'}$ equals $|\beta|/(2\sqrt2\alpha\sigma_0)$, and this is $h(r)$ with $\iota=1/\sigma_0^2$.

For the second claim, $\alpha(1)=0$ and $\alpha\in C^1$ give $\alpha(r)\le\norm{\alpha'}_\infty(1-r)$, with $\norm{\alpha'}_\infty>0$ because $\alpha(0)=1$. Continuity and $\beta(1)=1$ give $r_1<1$ with $\beta\ge1/2$ on $[r_1,1]$. Therefore
\[
 \int_0^1\frac{\beta(r)^2}{\alpha(r)^2}dr\ge\frac1{4\norm{\alpha'}_\infty^2}\int_{r_1}^1\frac{dr}{(1-r)^2}=\infty.\qquad\square
\]

\subsubsection{Admissible Flow-Time Laws}\label{app:timelaws}
\paragraph{Truncated uniform.} For the linear path, $J\equiv1$. With $\lambda=\mathrm{Unif}[0,1-\eta_0]$ and the substitution $u=1-r$,
\[
 m_k=\frac1{1-\eta_0}\int_{\eta_0}^1\frac{du}{u^2}=\frac1{\eta_0},\qquad
 m_h=\frac1{1-\eta_0}\int_{\eta_0}^1\frac{(1-u)^2}{u^2}du=\frac{\eta_0^{-1}-2\log(1/\eta_0)-\eta_0}{1-\eta_0}.
\]
For the sine--cosine path, $J\equiv\pi/2$, and $\int_0^{1-\eta_0}\sec^2(\pi r/2)dr=\frac2\pi\cot(\pi\eta_0/2)$. Hence
\[
 m_k=\frac{\pi\cot(\pi\eta_0/2)}{2(1-\eta_0)},\qquad
 m_h=\frac{\frac2\pi\cot(\pi\eta_0/2)-(1-\eta_0)}{1-\eta_0}.
\]
Both are of order $1/\eta_0$, so $\Lsm=O(\eta_0^{-1/2}\sqrt{\log(emT)})$.

\paragraph{Beta and logit-normal laws.} For the linear path and $\lambda=\mathrm{Beta}(a,b)$, $m_k=\E(1-r)^{-2}=B(a,b-2)/B(a,b)$, finite exactly when $b>2$, and $m_h\le m_k$. If $\mathrm{logit}(r)\sim N(\mu_\ell,s^2)$, then $(1-r)^{-2}=(1+e^{Y})^2\le2+2e^{2Y}$ with $Y=\mathrm{logit}(r)$, so $m_k\le2+2e^{2\mu_\ell+2s^2}$. For the sine--cosine path, $\alpha(r)=\sin(\pi(1-r)/2)\ge1-r$ by concavity, so the same bounds hold after multiplying $m_k$ by $\pi^2/4$.

%
\subsubsection{Historical Transfers}
\begin{lemma}[Target-law transfer]\label{lem:targettransfer}
Let $H=[a,a+q-1]$ precede round $t$. Then \eqref{eq:histdrift-target} and \eqref{eq:currentdrift-target} hold for every $v\in\V$, pointwise before expectation over any data.
\end{lemma}
\begin{proof}
For $s\in H$, the triangle inequality for $W_2$ gives $W_2(\mu_s,\mu_t)\le V^\mu_{a:t}$. By Lemma~\ref{lem:smoothing}, hypothesis \eqref{eq:hyp-second} holds for $(s,t)$ with $\hbar\le\Lsm V^\mu_{a:t}$ and $\rho\le\bar\rho_{a:t}$, since both remainders increase with the distance.

\emph{Historical approximation.} Because regression-noise offsets do not depend on the field, $v_H^\star$ also minimizes $q^{-1}\sum_{s\in H}Q_s$, as in Lemma~\ref{lem:drift}. Therefore
\[
 \bar a_H\le\frac1q\sum_{s\in H}Q_s(v_t^\star)=\frac1q\sum_{s\in H}[a_s+F_s(v_t^\star)],
\]
and Lemma~\ref{lem:secondorder}(i) bounds each $F_s(v_t^\star)$ by $5\Lsm^2(V^\mu_{a:t})^2+8(a_s+a_t)+2\bar\rho_{a:t}$. This proves \eqref{eq:histdrift-target}.

\emph{Current excess.} Average Lemma~\ref{lem:secondorder}(ii) over $s\in H$:
\[
 F_t(v)\le\frac3q\sum_{s\in H}F_s(v)+14\Lsm^2(V^\mu_{a:t})^2+\frac{12}q\sum_{s\in H}a_s+10a_t+4\bar\rho_{a:t}.
\]
By definition, $q^{-1}\sum_sF_s(v)=r_H(v)+q^{-1}\sum_s[L_s(v_H^\star)-L_s(v_s^\star)]$. Since $\bar L_H(v_H^\star)\le\bar L_H(v_t^\star)$, the bracket average is at most $q^{-1}\sum_sF_s(v_t^\star)\le5\Lsm^2(V^\mu_{a:t})^2+8q^{-1}\sum_sa_s+8a_t+2\bar\rho_{a:t}$. Substituting gives \eqref{eq:currentdrift-target}: $3\cdot5+14=29$, $3\cdot8+12=36$, $3\cdot8+10=34$, and $3\cdot2+4=10$.
\end{proof}

\subsubsection{Historical Fits over a Lifetime}
For a lifetime $I=[a,a+n_I-1]$, let $V_I^\mu=\sum_{s=a}^{a+n_I-2}W_2(\mu_{s+1},\mu_s)$ be its internal target-law variation, and recall $e_M=2\eopt+5\zeta+10\delta$ from Appendix~\ref{app:algorithm}.
\begin{lemma}\label{lem:basetarget}
The fitted model of every lifetime satisfies
\begin{align}
 J_I\le{}&B+\delta+\frac{216B\Gamma}m\log(en_I)+(180\log(en_I)+82)\mathcal A_I\nonumber\\
 &+(n_I-1)\big[59\Lsm^2(V_I^\mu)^2+22(\esm V_I^\mu+c_2(V_I^\mu)^2)+3e_M\big].\label{eq:basetarget}
\end{align}
\end{lemma}
\begin{proof}
At $t=a$, the deployed field $v_{I,a}\in\V$ satisfies $F_a(v_{I,a})\le L_a(v_{I,a})\le B+\delta$. At a later round, let $k=t-a$ and $q=2^{\lfloor\log_2k\rfloor}$, so $k/2\le q\le k$ and the prefix $H=[a,a+q-1]$ is deterministic. Apply the current-risk transfer from Lemma~\ref{lem:targettransfer}, Equation~\eqref{eq:currentdrift-target}, pointwise to the fit, and then the historical estimation bound of Proposition~\ref{thm:historical}, Equation~\eqref{eq:histfast}, proved in Appendix~\ref{app:common}. This gives
\[
 \E F_t(v_{I,t})\le3\Big[\frac{36B\Gamma}{mq}+2\bar a_H+e_M\Big]+29\Lsm^2(V^\mu_{a:t})^2+\frac{36}q\sum_{s\in H}a_s+34a_t+10\bar\rho_{a:t}.
\]
Insert the historical-approximation bound from Lemma~\ref{lem:targettransfer}, Equation~\eqref{eq:histdrift-target}, for $\bar a_H$, which is deterministic under the fixed environment:
\[
 \E F_t(v_{I,t})\le\frac{108B\Gamma}{mq}+\frac{90}q\sum_{s\in H}a_s+82a_t+59\Lsm^2(V^\mu_{a:t})^2+22\bar\rho_{a:t}+3e_M.
\]
Now sum over $k=1,\ldots,n_I-1$. As in Lemma~\ref{lem:base}, $\sum_k1/q\le2\log(en_I)$, so the statistical terms sum to at most $216B\Gamma\log(en_I)/m$ and each $a_s$ receives total coefficient at most $180\log(en_I)$ from the historical sums. The current terms contribute $82\mathcal A_I$. For $t\in I$, $V^\mu_{a:t}\le V_I^\mu$, so the drift terms contribute at most $(n_I-1)[59\Lsm^2(V_I^\mu)^2+22(\esm V_I^\mu+c_2(V_I^\mu)^2)]$, and the residuals $(n_I-1)3e_M$.
\end{proof}

\subsubsection{Interval Bound}
\begin{lemma}\label{lem:intervaltarget}
For every lifetime $I$,
\begin{equation}\label{eq:dyadictarget}
 S_I\le K_T^+/D_T+H_T^+\mathcal A_I+n_Ie_f^++(n_I-1)\big[99\Lsm^2(V_I^\mu)^2+37(\esm V_I^\mu+c_2(V_I^\mu)^2)\big],
\end{equation}
and for every consecutive interval $J\subseteq[1,T]$ inequality \eqref{eq:intervaltarget} holds.
\end{lemma}
\begin{proof}
Insert the target-law lifetime estimate of Lemma~\ref{lem:basetarget}, Equation~\eqref{eq:basetarget}, into the shared fast master bound of Lemma~\ref{lem:fastmaster}, Equation~\eqref{eq:fastmasterbound}. The non-statistical constant $(5/3)B+(128/3)B\log M_e$ is at most $44B\Lambda$, and $(5/3)216=360$ gives the statistical term of $K_T^+/D_T$. The approximation coefficient is $(5/3)(180\log(en_I)+82)+2/3\le300\log(eT)+138=H_T^+$. The drift coefficients are $(5/3)59\le99$ and $(5/3)22\le37$. The residuals are $(5/3)(n_I-1)3e_M+(5/3)\delta+(8/3)n_I\delta\le n_I(5e_M+13\delta/3)\le n_Ie_f^+$. This proves \eqref{eq:dyadictarget}.

For an interval $J$, take the disjoint dyadic cover of Lemma~\ref{lem:dyadiccover}: at most $D_T$ lifetimes, whose internal edges are internal edges of $J$. Then $\sum_I(n_I-1)(V_I^\mu)^2\le(n_J-1)(\sum_IV_I^\mu)^2\le(n_J-1)(V_J^\mu)^2$ and $\sum_I(n_I-1)V_I^\mu\le(n_J-1)V_J^\mu$. Summing \eqref{eq:dyadictarget} and using $\bar c=99\Lsm^2+37c_2$ gives \eqref{eq:intervaltarget}.
\end{proof}

\subsection{Proof of Theorem 2}\label{sec:proof-theorem2}
\subsubsection{Total-Variation Partition}

For the variation branch we use Lemma~\ref{lem:intervaltarget}, proved above by combining Lemma~\ref{lem:basetarget} with the shared fast master bound of Lemma~\ref{lem:fastmaster} and the dyadic cover of Lemma~\ref{lem:dyadiccover}. The target-law lifetime estimate itself uses Lemmas~\ref{lem:targettransfer} and~\ref{lem:secondorder}. For the count branch we use the first-order lifetime bound of Lemma~\ref{lem:base} with the same shared aggregation and covering lemmas, as established in Appendix~\ref{app:regret}.

\textbf{Count part.} By Remark~\ref{rem:timelaw}, the count bound of Theorem~\ref{thm:main} holds under the flow-time law of Assumption~\ref{ass:smooth}. Since $N_T^\omega\le N_T^\mu\le N_\mu$ (Section~\ref{sec:budgets}), it gives the first entry of Equation~\eqref{eq:target}.

\textbf{Variation part.} Let $\tau_0=K_T^+/\bar c$. Construct pieces $J_1,J_2,\ldots,J_{M'}$ from left to right: each starts after the previous one ends and is the longest interval $J$ with $|J|(V_J^\mu)^2\le\tau_0$. A singleton has zero internal variation, so every piece is nonempty and the construction terminates at $T$. Let $n_i=|J_i|$. For $i<M'$, extending $J_i$ by the next round violates the defining inequality. The extension has length $n_i+1\le2n_i$ and internal variation $V_i^+=V_{J_i}^\mu+W_2(\mu_{e_i+1},\mu_{e_i})$, where $e_i$ is the last round of $J_i$. Hence $V_i^+>\sqrt{\tau_0/(2n_i)}$. The edges counted by the $V_i^+$ are disjoint, so $\sum_{i<M'}V_i^+\le V_\mu$. By convexity of $x\mapsto x^{-1/2}$ and $\sum_{i<M'}n_i\le T$,
\[
 V_\mu>\sqrt{\frac{\tau_0}2}\sum_{i<M'}n_i^{-1/2}\ge\sqrt{\frac{\tau_0}2}\,(M'-1)^{3/2}T^{-1/2},
\]
so $M'\le1+(2\bar cTV_\mu^2/K_T^+)^{1/3}$. If $V_\mu=0$, then $M'=1$. On each piece, the interval bound of Lemma~\ref{lem:intervaltarget}, Equation~\eqref{eq:intervaltarget}, and $(n_i-1)\bar c(V_{J_i}^\mu)^2\le\bar c\tau_0=K_T^+$ bound the cost by $2K_T^++H_T^+\mathcal A_{J_i}+n_ie_f^++37\esm n_iV_{J_i}^\mu$. Summing, and using $\sum_in_iV^\mu_{J_i}\le TV_\mu$,
\[
 \E[\Reg\mid E]\le2K_T^+M'+H_T^+\App+Te_f^++37\esm TV_\mu,
\]
which is the second entry of Equation~\eqref{eq:target} after inserting the bound on $M'$.

\textbf{Integration.} The caps and constants are deterministic, so both conditional bounds integrate over $E$, with $\App$ replaced by $\E\App$, and the minimum of the two unconditional bounds then bounds $\E\Reg$. The supplementary uniform estimate also follows from $F_t(v)\le L_t(v)\le B+\delta$, giving $\E\Reg\le T(B+\delta)$. It is separate from the count and variation guarantee stated in the theorem.

Thus, under the predictive protocol, Assumptions~\ref{ass:tails},~\ref{ass:network},~\ref{ass:fit},~\ref{ass:smooth}, and the caps in Equation~\eqref{eq:targetvariation}, Algorithm~\ref{alg:nsfm} satisfies
\begin{align}
 \E\Reg\le{}&\Targ,\nonumber\\
 \Targ={}&\min\Big\{K_T(N_\mu+1)+H_T\E\App+Te_f,\nonumber\\
 &\quad 2K_T^++2(2\bar c)^{1/3}(K_T^+)^{2/3}T^{1/3}V_\mu^{2/3}\nonumber\\
 &\quad+37\esm TV_\mu+H_T^+\E\App+Te_f^+\Big\}.\label{eq:proof-target-conclusion}
\end{align}
Hence Theorem~\ref{thm:target} is proved. \hfill$\square$

\supppage

\section{Minimax Lower Bounds}\label{app:lower}
This appendix proves the lower bounds summarized after Theorem~\ref{thm:target}. A lower bound holds \emph{at fixed scales} if every hard instance uses one class and satisfies the assumptions with the same constants $d,m,\kappa,c_Z,\iota,\lambda$ and the same architectural bounds, all independent of $T$ and the budget, so that only the target sequence varies. All three constructions meet this requirement. The counting and target-law bounds use the Gaussian class of Appendix~\ref{app:gaussian}. The relative-risk bound uses the class $\V_\star$ of Appendix~\ref{sec:rarefamily}, whose trainable slope lets its cutoff follow a far target component that moves out like $\sqrt{\log(1/\eta)}$, as in the fixed-target construction of \citet{sahoo2026}. The bounds are for proper learners, which deploy fields in the fixed class. Every construction draws signs at random and bounds the regret averaged over the signs, so for every learner some sign sequence attains the bound. This is the order of quantifiers in Propositions~\ref{prop:lb-count}--\ref{prop:lb-target}. Table~\ref{tab:rates} collects the resulting rates. All constructions use the sine--cosine schedules $\alpha(r)=\cos(\pi r/2)$, $\beta(r)=\sin(\pi r/2)$, a standard Gaussian base, independent pairing, and fixed $m$. Write $\phi_N$ and $\Phi$ for the standard normal density and distribution function. A \emph{learner} is any rule that deploys $v_t\in\V$ as a measurable function of earlier batches and independent seeds.

\begin{table}[h]
\centering
\caption{Leading rates at fixed scales, for one fixed class per construction, logarithms suppressed. Upper bounds additionally pay $O(\log T)\E\App$ and $T$ times the fitting and proxy residuals.}
\label{tab:rates}
\small
\begin{tabular}{@{}lllll@{}}
\toprule
Variation & Budget & Upper bound & Lower bound & Family\\
\midrule
Relative risk & $N_\omega$ & $N_\omega+1$ (Thm.~\ref{thm:main}) & $N_\omega+1$ (Prop.~\ref{prop:lb-count}) & Gaussian means\\
Relative risk & $\Omega$ & $1+\sqrt{T\Omega}$ (Thm.~\ref{thm:main}) & $\sqrt{T\Omega/m}$ (Prop.~\ref{prop:lb-omega}) & rare Gaussian mixture\\
Target law & $N_\mu$ & $N_\mu+1$ (Thm.~\ref{thm:target}) & $N_\mu+1$ (Prop.~\ref{prop:lb-count}) & Gaussian means\\
Target law & $V_\mu$ & $1+T^{1/3}V_\mu^{2/3}$ (Thm.~\ref{thm:target}) & $T^{1/3}V_\mu^{2/3}m^{-2/3}$ (Prop.~\ref{prop:lb-target}) & Gaussian means\\
\bottomrule
\end{tabular}
\end{table}

\subsection{Information Structure and a Two-Point Reduction}
\begin{lemma}[Sine--cosine triples reveal both endpoints]\label{lem:information}
With $J=\alpha\beta'-\alpha'\beta\equiv\pi/2$, $X_0=(\beta'I_r-\beta Z_r)/J$ and $X_1=(\alpha Z_r-\alpha'I_r)/J$. Consequently, for two target sequences that agree before round $t_0$ and whose round laws differ only on rounds in $[t_0,t)$, the laws $P,P'$ of the learner's information before round $t$ satisfy $\mathrm{TV}(P,P')\le\mathrm{TV}(\otimes_i\mu_{(i)},\otimes_i\mu'_{(i)})$, where the products run over the at most $m(t-t_0)$ target draws of those rounds.
\end{lemma}
\begin{proof}
The identities follow from $\beta'I-\beta Z=(\alpha\beta'-\alpha'\beta)X_0$ and $\alpha Z-\alpha'I=(\alpha\beta'-\alpha'\beta)X_1$. Hence the information is a measurable function of the target draws together with base draws, flow times, earlier-round draws, and seeds, all of which have the same law under both sequences and are independent of the differing draws. Total variation does not increase under such a map.
\end{proof}

\begin{lemma}[Two-point reduction]\label{lem:twopoint}
Let $Y$ be the learner's information and seed before a round, with law $P_+$ or $P_-$, and let the deployed field be a measurable function of $Y$. For nonnegative functions $f_\pm$ on $\V$,
\[
 \tfrac12\E_+f_+(v)+\tfrac12\E_-f_-(v)\ge\tfrac12\big(1-\mathrm{TV}(P_+,P_-)\big)\inf_{u\in\V}\{f_+(u)+f_-(u)\}.
\]
\end{lemma}
\begin{proof}
With densities $p_\pm$ for a dominating measure, $\E_+f_++\E_-f_-=\int(f_+p_++f_-p_-)\ge\int(f_++f_-)\min(p_+,p_-)$, and $\int\min(p_+,p_-)=1-\mathrm{TV}(P_+,P_-)$.
\end{proof}
All constructions below draw block signs independently and uniformly. A lower bound on the expected regret under this prior implies the same bound for some deterministic sign sequence, because the prior has finite support. The environment is therefore exogenous and the caps hold for every sign sequence.

\subsection{The Gaussian Family}\label{app:gaussian}
Fix $R_0>0$ and an exogenous sequence of means $\xi_t\in\R^d$ with $\norm{\xi_t}\le R_0$. Let $\nu=N(0,I_d)$, $\mu_t=N(\xi_t,I_d)$, $I_r=\alpha(r)X_0+\beta(r)X_1$, $Z_r=\alpha'(r)X_0+\beta'(r)X_1$, and $w(r)=\beta'(r)=\frac\pi2\cos(\pi r/2)$. For $\vartheta\in\R^d$ with $\norm\vartheta\le R_0$, let $v_\vartheta(x,r)=w(r)\vartheta$. For a flow-time law $\lambda$ put $c_*^\lambda=\E_\lambda[w(r)^2]$, which equals $c_*=\pi^2/8$ for the uniform law.

\begin{lemma}[Gaussian family]\label{lem:gaussfamily}
For every flow-time law $\lambda$:
\begin{enumerate}
\item[(i)] Assumption~\ref{ass:tails} holds with $\kappa^2=8(d+R_0^2+1)$, Assumption~\ref{ass:path} with $c_Z=\pi/2$, and Assumption~\ref{ass:smooth}(a)--(b) with $\iota=1$,
\item[(ii)] the class $\{v_\vartheta\}$ satisfies Assumption~\ref{ass:network} with one hidden unit and the saturating sine
\begin{equation}\label{eq:saturating-sine}
 \phi(z)=-1\ (z<-\tfrac\pi2),\qquad \phi(z)=\sin z\ (|z|\le\tfrac\pi2),\qquad \phi(z)=1\ (z>\tfrac\pi2);
\end{equation}
\item[(iii)] $b_t(x,r)=w(r)\xi_t$ for $\lambda$-almost every $r$, $Q_t(v_\vartheta)=c_*^\lambda\norm{\vartheta-\xi_t}^2$, and $a_t=0$,
\item[(iv)] $\omega_t=2c_*^\lambda R_0\norm{\xi_{t+1}-\xi_t}$ and $W_2(\mu_{t+1},\mu_t)=\norm{\xi_{t+1}-\xi_t}$.
\end{enumerate}
\end{lemma}
\begin{proof}
\emph{(i)} For $X\sim N(\xi,I_d)$ and $\kappa^2>2$, completing the square gives $\E e^{\norm X^2/\kappa^2}=(1-2/\kappa^2)^{-d/2}\exp(\norm\xi^2/(\kappa^2-2))$. By $-\log(1-u)\le u/(1-u)$, its logarithm is at most $(d+R_0^2)/(\kappa^2-2)<1/8<\log2$ for $\kappa^2=8(d+R_0^2+1)$, for base and targets alike. The schedules are smooth with correct endpoints and $|\alpha'|,|\beta'|\le\pi/2$. The Gaussian base has Fisher information $I_d$.

\emph{(ii)} The function $\phi$ is bounded by one, nondecreasing, continuously differentiable (its interior derivative vanishes at the junctions), and $1$-Lipschitz. Fix the first affine map to $z=\frac\pi2(1-r)$ with zero spatial weights. For $r\in[0,1]$ the hidden value is $\phi(z)=\sin(\frac\pi2(1-r))=\cos(\pi r/2)$. With output matrix $\frac\pi2\vartheta$ and zero output bias, the output is $w(r)\vartheta$. The trainable parameter is the output column, a compact ball.

\emph{(iii)} Write $X_0=G_0$ and $X_1=\xi_t+G_1$ with independent standard Gaussian vectors. At a fixed flow time, $I_r$ has mean $\beta\xi_t$ and covariance $(\alpha^2+\beta^2)I_d=I_d$, $Z_r$ has mean $\beta'\xi_t$, and $\operatorname{Cov}(Z_r,I_r)=(\alpha\alpha'+\beta\beta')I_d=0$ because $\alpha^2+\beta^2\equiv1$. Joint Gaussianity makes the centered label independent of the input, so $b_t(x,r)=w(r)\xi_t$ at every fixed time, hence for $\lambda$-almost every $r$. Then $Q_t(v_\vartheta)=\E_\lambda[w(r)^2]\norm{\vartheta-\xi_t}^2$, and $\vartheta=\xi_t$ is in the class.

\emph{(iv)} With $\Delta_t=\xi_{t+1}-\xi_t$, $[Q_{t+1}(v_\vartheta)-Q_{t+1}(v_{\vartheta'})]-[Q_t(v_\vartheta)-Q_t(v_{\vartheta'})]=-2c_*^\lambda\langle\vartheta-\vartheta',\Delta_t\rangle$. The difference of two radius-$R_0$ balls is the radius-$2R_0$ ball, so half the supremum of the absolute value is $2c_*^\lambda R_0\norm{\Delta_t}$. Finally, translation by $\Delta_t$ is an optimal coupling of two Gaussians with equal covariance, and Jensen's inequality bounds every coupling cost below by $\norm{\Delta_t}^2$.
\end{proof}
The relative-risk variation of this family is thus $2c_*^\lambda R_0$ times its target-law variation, and a stale optimum loses exactly $F_t(v_{\xi_s})=c_*^\lambda\norm{\xi_t-\xi_s}^2$.
For a history $H$ of $q$ rounds, let $\bar\xi_H=q^{-1}\sum_{s\in H}\xi_s$. This mean is admissible because the radius-$R_0$ ball is convex. Lemma~\ref{lem:gaussfamily}(iii) makes $v_{\bar\xi_H}$ the best common field, so
\[
 \bar a_H=\frac{c_*^\lambda}{q}\sum_{s\in H}\norm{\xi_s-\bar\xi_H}^2,
\]
although $a_s=0$ for every round. Thus historical pooling can create a positive approximation residual even in a realizable family.

\subsection{Counting Lower Bounds}
\begin{proposition}[Counting lower bound]\label{prop:lb-count}
For every flow-time law $\lambda$, every $N<T$, and every learner, there is a mean sequence for the Gaussian family with $d=1$ and $N_T^\omega=N_T^\mu\le N$ on which $\E\Reg\ge c_*^\lambda R_0^2(N+1)$.
\end{proposition}
\begin{proof}
Use targets $N(R_0\zeta_j,1)$ on $K=N+1$ consecutive blocks of lengths at least one, where the signs $\zeta_j$ are independent and uniform. By Lemma~\ref{lem:gaussfamily}(iii), $F_t(v_\vartheta)=c_*^\lambda(\vartheta-\xi_t)^2$. At the first round of block $j$ the learner's information is independent of $\zeta_j$, so Lemma~\ref{lem:twopoint} with zero total variation gives
\[
 \E F_t(v_t)\ge\tfrac12\inf_\vartheta c_*^\lambda[(\vartheta-R_0)^2+(\vartheta+R_0)^2]=c_*^\lambda R_0^2.
\]
Summing over blocks gives $\E\Reg\ge c_*^\lambda R_0^2K$. Target changes occur only at block boundaries, so $N_T^\mu\le K-1$, and $N_T^\omega=N_T^\mu$ by Lemma~\ref{lem:gaussfamily}(iv).
\end{proof}

\subsection{The Rare-Mixture Family for Relative-Risk Variation}\label{sec:rarefamily}
The two-point construction of \citet{sahoo2026} hides the difference between two targets in a far component of small mass: total variation of order $\eta$, velocity separation of order $R^2\eta$. We use it block by block, with the far mass as the drift parameter. Regret is measured against a class, so the class must see the far component while staying fixed. A trainable slope achieves this: shrinking the slope moves the cutoff of the field into the tails, while every weight and bias bound stays fixed.

\paragraph{One fixed class.} With the saturating sine $\phi$ of Equation~\eqref{eq:saturating-sine} and $b_\circ=11\pi/2$, let
\begin{equation}\label{eq:fixedclass}
 \psi_\gamma(x)=1+\tfrac12\phi(\gamma x-b_\circ)+\tfrac12\phi(-\gamma x-b_\circ),\qquad
 \V_\star=\{v_{\theta,\gamma}(x,r)=\theta\psi_\gamma(x):|\theta|\le1,\ 0\le\gamma\le4\}.
\end{equation}
One of the two arguments is always at most $-b_\circ<-\pi/2$, so $\psi_\gamma(x)=\frac12[1+\phi(\gamma|x|-b_\circ)]$. Thus $\psi_\gamma$ is even, $0\le\psi_\gamma\le1$, $\psi_0=0$, and for $\gamma>0$
\begin{equation}\label{eq:cutoffs}
 \psi_\gamma(x)=0\ \text{ if }|x|\le5\pi/\gamma,\qquad\psi_\gamma(x)=1\ \text{ if }|x|\ge6\pi/\gamma.
\end{equation}

\paragraph{Instance.} Let $\lambda=\mathrm{Unif}[0,1]$, fix $\eta\in(0,e^{-12}]$, put $L_\eta=\log(1/\eta)$ and $R=\sqrt{8L_\eta}$, and let
\begin{equation}\label{eq:raremix}
 \mu_s=(1-\eta)N(0,1)+\eta N(sR,1),\qquad s\in\{-1,+1\}.
\end{equation}
Only the targets vary with $\eta$. The class remains $\V_\star$.

\begin{lemma}[Fixed class and constants]\label{lem:rarefixed}
For every $\eta$, Assumption~\ref{ass:tails} holds with $\kappa^2=18$ and Assumption~\ref{ass:path} with $c_Z=\pi/2$. The class $\V_\star$ satisfies Assumption~\ref{ass:network} with two affine layers of widths $(2,2,1)$, $P=5$ trainable entries in a compact set of radius $R_\Theta=\sqrt{67/2}$, layer bounds $V_1=4\sqrt2$, first-layer bias norm $11\pi/\sqrt2$, $V_2=1/\sqrt2$, output bias bound $1$, $B_v=2$, and $L_x=4$. On $|x|\le S$ its parameter sensitivity is at most $C_S=(\sqrt{S^2+1}+1)/\sqrt2+\sqrt2+1$. None of these constants depends on $\eta$, $T$, $m$, or the budget.
\end{lemma}
\begin{proof}
For $X\sim N(\xi,1)$, $\E e^{X^2/18}=(8/9)^{-1/2}e^{\xi^2/16}$. For the base this is $1.061<2$. For Equation~\eqref{eq:raremix}, $e^{R^2/16}=\eta^{-1/2}$, so $\E e^{X^2/18}=(8/9)^{-1/2}(1-\eta+\sqrt\eta)\le1.061\cdot1.25<2$. The sine--cosine schedules have $|\alpha'|,|\beta'|\le\pi/2$. For input $(x,r)$ take
\[
 W_1=\begin{pmatrix}\gamma&0\\-\gamma&0\end{pmatrix},\qquad b_1=\begin{pmatrix}-b_\circ\\-b_\circ\end{pmatrix},\qquad W_2=(\theta/2,\theta/2),\qquad b_2=\theta,
\]
with the activation $\phi$, which is bounded by one, $1$-Lipschitz, $C^1$, and nondecreasing. The output is $\theta+\frac\theta2(\phi_1+\phi_2)=\theta\psi_\gamma$. The trainable entries form $\Theta_\star=\{(\gamma,-\gamma,\theta/2,\theta/2,\theta):\gamma\in[0,4],\theta\in[-1,1]\}\subset\R^5$, a compact set with squared radius $2\cdot4^2+3/2=67/2$. The norms give $V_1=4\sqrt2$, $\norm{b_1}=11\pi/\sqrt2$, $V_2=1/\sqrt2$, and $|b_2|\le1$, so $B_v=V_2\sqrt2+1=2$ and $L_x=V_2V_1=4$. A first-layer perturbation changes the output by at most $V_2\sqrt{S^2+1}$ times its norm, and an output-layer perturbation by at most $\sqrt2+1$ times its norm, which gives $C_S$. It grows only with the input radius, as the gated cover of Corollary~\ref{cor:cover} requires.
\end{proof}

For $\xi\sim N(0,1)$ and $s\in\{-1,+1\}$ let
\[
 N_\eta(\gamma)=\E_s[\psi_\gamma(I_r)^2],\qquad G_\eta(\gamma)=\int_0^1\beta'(r)\,\E_\xi\psi_\gamma(\beta(r)R+\xi)\,dr .
\]
Reflection $y\mapsto-y$ maps $\mu_+$ to $\mu_-$ and fixes $\psi_\gamma$, so $N_\eta(\gamma)$ does not depend on $s$. Also $0\le G_\eta(\gamma)\le\int_0^1\beta'=1$.

\begin{lemma}[Risk geometry over the whole class]\label{lem:raremix}
Let $\eta\le e^{-12}$, and let $p_s$ be the posterior probability of the far component.
\begin{enumerate}
\item[(i)] $b_s(y,r)=\beta'(r)sR\,p_s(y,r)$, and there is a constant $C_\eta$, independent of $s,\theta,\gamma$, with $L_s(v_{\theta,\gamma})=C_\eta+\theta^2N_\eta(\gamma)-2s\theta\eta RG_\eta(\gamma)$.
\item[(ii)] $\inf_{v\in\V_\star}\{F_+(v)+F_-(v)\}\ge\frac12\eta R$.
\item[(iii)] If consecutive rounds have opposite signs, then $\omega_t=4\eta R\max_{0\le\gamma\le4}G_\eta(\gamma)\le4\eta R$. Equal signs give $\omega_t=0$.
\item[(iv)] $a_s\le\frac{\pi^2}4R^2\eta$.
\end{enumerate}
\end{lemma}
\begin{proof}
\emph{(i)} For a target component $N(c,1)$ and a fixed $r$, $(I_r,Z_r)$ is jointly Gaussian with $\mathrm{Var}(I_r)=\alpha^2+\beta^2=1$ and $\mathrm{Cov}(I_r,Z_r)=\alpha\alpha'+\beta\beta'=0$, so its conditional mean label is $\beta'c$. Conditioning on the component gives $b_s=\beta'sRp_s$, because the bulk has $c=0$. The bulk has mean label zero, and $\psi_\gamma$ is even, so the tower property gives $\E_s[\psi_\gamma(I_r)Z_r]=s\eta RG_\eta(\gamma)$. Expanding the square proves (i), with $C_\eta=\E_sZ_r^2$, which is the same for both signs.

\emph{(ii)} By compactness and dominated convergence the minima below exist. Let $d_\eta=\min_{|\theta|\le1,\,0\le\gamma\le4}\{\theta^2N_\eta(\gamma)-2\theta\eta RG_\eta(\gamma)\}$. Replacing $\theta$ by $-\theta$ exchanges the signs, so the minimum loss for either sign is $C_\eta+d_\eta$. Hence, for every admissible $(\theta,\gamma)$,
\begin{equation}\label{eq:sumexcess}
 F_+(v_{\theta,\gamma})+F_-(v_{\theta,\gamma})=2\theta^2N_\eta(\gamma)-2d_\eta\ge-2d_\eta.
\end{equation}
To bound $d_\eta$, use the witness slope $\gamma_\eta=10\pi/R$, which lies in $[0,4]$ because $R\ge\sqrt{96}$. By \eqref{eq:cutoffs} its field vanishes on $|x|\le R/2=\sqrt{2L_\eta}$ and equals one on $|x|\ge3R/5$. Under the bulk, $I_r\sim N(0,1)$ for every $r$, so Mills' inequality gives
\[
 N_\eta(\gamma_\eta)\le\Pr(|\xi|>\sqrt{2L_\eta})+\eta\le\eta\big(1+(\pi L_\eta)^{-1/2}\big)\le\tfrac76\eta .
\]
On $\{r:\beta(r)\ge3/5\}$ the event $\xi\ge0$ puts $\beta R+\xi$ on the unit plateau. Since $\beta'\ge0$ and $\psi_\gamma\ge0$, $G_\eta(\gamma_\eta)\ge\frac12\int_{\{\beta\ge3/5\}}\beta'(r)\,dr=\frac15$. Evaluating the minimum at $(1,\gamma_\eta)$,
\[
 -2d_\eta\ge4\eta RG_\eta(\gamma_\eta)-2N_\eta(\gamma_\eta)\ge\eta\Big(\frac{4R}5-\frac73\Big)\ge\frac12\eta R,
\]
where the last step uses $R\ge\sqrt{96}>8$. With \eqref{eq:sumexcess} this proves the separation for the whole class. The witness need not be the global optimizer.

\emph{(iii)} By (i), $L_-(v_{\theta,\gamma})-L_+(v_{\theta,\gamma})=4\theta\eta RG_\eta(\gamma)$. Over the class this ranges over the symmetric interval with endpoints $\pm4\eta R\max_\gamma G_\eta(\gamma)$, and $\omega_t$ is half its oscillation. Equal signs give equal risks.

\emph{(iv)} Since $0\in\V_\star$, $p_s^2\le p_s$, and $\E_sp_s(I_r,r)=\eta$, we get $a_s\le Q_s(0)=\E_sb_s(I_r,r)^2\le\frac{\pi^2}4R^2\E_sp_s(I_r,r)^2\le\frac{\pi^2}4R^2\eta$.
\end{proof}

\begin{proposition}[Relative-risk variation lower bound]\label{prop:lb-omega}
Fix the class $\V_\star$, the standard Gaussian base, the sine--cosine path, and $\lambda=\mathrm{Unif}[0,1]$. Let $T\ge2$ and $mT\ge e^{12}/2$, and let $K_{\max}=\lfloor\min\{T/2,2mTe^{-12}\}\rfloor$ and $R_K=\sqrt{8\log(4m\lfloor T/K\rfloor)}$. For every $\Omega$ with
\begin{equation}\label{eq:omegarange}
 \frac{2R_1}{mT}\le\Omega\le\frac{2K_{\max}^2R_{K_{\max}}}{mT},
\end{equation}
and every proper learner in $\V_\star$, there is a sign sequence for the family in Equation~\eqref{eq:raremix} with $\Omega_T\le\Omega$ on which $\E\Reg\ge\frac1{20}\sqrt{T\Omega/m}$. The class, the envelope, the flow-time law, and every architectural bound are independent of $T$, $m$, and $\Omega$.
\end{proposition}
\begin{proof}
Fix $K\le K_{\max}$, let $H=\lfloor T/K\rfloor\ge T/(2K)$, and let $\eta=1/(4mH)$. The choice $K\le2mTe^{-12}$ gives $4mH\ge e^{12}$, that is, $\eta\le e^{-12}$. Use $K$ blocks of length $H$ with independent uniform signs, sampled independently of the learner, and let the remaining $T-KH$ rounds keep the last sign. At a round of block $j$, the two sign hypotheses for that block agree before it, so by Lemma~\ref{lem:information} the learner's information differs through at most $mH$ target draws, and $\mathrm{TV}\le mH\,\mathrm{TV}(\mu_+,\mu_-)\le mH\eta=1/4$. This holds even if the learner knows the block boundaries and the mixture parameters. Lemmas~\ref{lem:twopoint} and~\ref{lem:raremix}(ii) give $\E F_t(v_t)\ge\frac12\cdot\frac34\cdot\frac12\eta R=\frac3{16}\eta R$ at every block round, and summing over the $KH$ block rounds with $H\eta=1/(4m)$,
\[
 \E\Reg\ge\frac{3KR}{64m}.
\]
By Lemma~\ref{lem:raremix}(iii), there are at most $K-1$ sign changes, so $\Omega_T\le4(K-1)\eta R\le KR/(mH)\le2K^2R/(mT)$ almost surely.

Given $\Omega$ in the range \eqref{eq:omegarange}, let $K$ be the largest integer in $[1,K_{\max}]$ with $2K^2R_K/(mT)\le\Omega$. It exists because $K=1$ qualifies, and the instance then has $\Omega_T\le\Omega$. If $K<K_{\max}$, then $\Omega<2(K+1)^2R_{K+1}/(mT)\le8K^2R_K/(mT)$, because $H$, and hence $R_K$, does not increase with $K$. If $K=K_{\max}$, the upper end of the range gives the same inequality. In both cases $K\ge\sqrt{mT\Omega/(8R_K)}$, so
\[
 \E\Reg\ge\frac{3KR_K}{64m}\ge\frac3{64}\sqrt{\frac{R_K}8}\sqrt{\frac{T\Omega}m}\ge\frac1{20}\sqrt{\frac{T\Omega}m},
\]
using $R_K\ge\sqrt{96}$, which makes the middle coefficient at least $0.0518$. The bound holds on average over the block signs, so for every learner some sign sequence attains it. Lemma~\ref{lem:rarefixed} shows that every instance satisfies the assumptions with the stated constants.
\end{proof}

\begin{remark}[Scope of the relative-risk lower bound]
Proposition~\ref{prop:lb-omega} is a lower bound for proper learners in one fixed admissible class. It does not claim hardness for every larger neural class or for arbitrary measurable fields.
\end{remark}

\paragraph{Approximation and target-law budget of the family.} By Lemma~\ref{lem:raremix}(iv), $\App\le T\frac{\pi^2}4R^2\eta\le\frac{\pi^2}8KR^2/m$, and $W_2(\mu_+,\mu_-)\le2R\sqrt\eta$: couple the bulks identically and move the far mass $\eta$ by $2R$. The approximation cost of the upper bound is therefore $O(\log(eT)R)$ times the block lower bound $3KR/(64m)$, a polylogarithmic factor since $R=O(\sqrt{\log(emT)})$. The class has $P=5$ and $C_S=O(1+S)$, so the constants of Theorem~\ref{thm:main} on this family are polylogarithmic in $mT$ as well.

\subsection{The Gaussian Family for Target-Law Variation}
\begin{proposition}[Target-law variation lower bound]\label{prop:lb-target}
Let $\lambda$ satisfy Assumption~\ref{ass:smooth}(c), $R_0>0$, $T\ge2$, and
\[
 (2mT)^{-1/2}\le V_\mu\le\min\{T/(4\sqrt m),\,16R_0^3mT\}.
\]
For every learner there is a mean sequence for the Gaussian family with $\norm{\xi_t}\le R_0$ and $V_T^\mu\le V_\mu$ on which
\[
 \E\Reg\ge\frac{c_*^\lambda}{34}\,T^{1/3}V_\mu^{2/3}m^{-2/3}.
\]
\end{proposition}
\begin{proof}Use the Gaussian family with $d=1$. By Lemma~\ref{lem:gaussfamily} it satisfies Assumptions~\ref{ass:tails},~\ref{ass:network} and~\ref{ass:smooth} with $\kappa^2=8(2+R_0^2)$, $\iota=1$, and constants that do not depend on $\lambda$. Let $K=\lfloor(2mTV_\mu^2)^{1/3}\rfloor$. The lower end of the range gives $K\ge1$, and hence $K\ge\frac12(2mTV_\mu^2)^{1/3}$. The upper end gives $K\le T/2$ and $K\le8R_0^2mT$. Let $H=\lfloor T/K\rfloor\ge T/(2K)$ and $\delta_0=1/(4\sqrt{mH})$, so $\delta_0\le R_0$. Use $K$ blocks of length $H$ with means $\xi_t=\delta_0\zeta_j$ for every round $t$ of block $j$ and independent uniform signs $\zeta_j$, the remainder keeping the last mean.

At a round of block $j$, the information differs through at most $mH$ draws. By Lemma~\ref{lem:information},
\[
 \mathrm{KL}\big(N(\delta_0,1)^{\otimes mH}\,\|\,N(-\delta_0,1)^{\otimes mH}\big)=2mH\delta_0^2=\tfrac18,
\]
and Pinsker's inequality gives $\mathrm{TV}\le1/4$. Since $F_t(v_\vartheta)=c_*^\lambda(\vartheta-\xi_t)^2$ and $(\vartheta-\delta_0)^2+(\vartheta+\delta_0)^2\ge2\delta_0^2$, Lemma~\ref{lem:twopoint} gives $\E F_t(v_t)\ge\frac34c_*^\lambda\delta_0^2$. Summing over $KH$ rounds,
\[
 \E\Reg\ge\tfrac34KHc_*^\lambda\delta_0^2=\tfrac3{64}c_*^\lambda K/m.
\]
The budget is $V_T^\mu\le2\delta_0(K-1)<K/(2\sqrt{mH})\le K^{3/2}/\sqrt{2mT}\le V_\mu$. Finally,
\[
 \E\Reg\ge\tfrac3{128}c_*^\lambda(2mTV_\mu^2)^{1/3}/m=\tfrac{3\cdot2^{1/3}}{128}c_*^\lambda T^{1/3}V_\mu^{2/3}m^{-2/3},
\]
and $3\cdot2^{1/3}/128\ge1/34$. The family is realizable, $a_t=0$.\end{proof}

\subsection{Matching the Upper Bounds}
\begin{corollary}[Optimality up to logarithmic factors]\label{cor:minimax}
Fix the scales and the gate choice in Equation~\eqref{eq:gatechoice}. On each family above, the upper bound of Theorem~\ref{thm:main} or~\ref{thm:target} exceeds the corresponding lower bound by at most a factor polylogarithmic in $mT$ and polynomial in $m$, plus the additive fitting cost $4T\eopt$ (respectively $10T\eopt$). Hence the rates $N_\omega+1$, $\sqrt{T\Omega}$, $N_\mu+1$, and $T^{1/3}V_\mu^{2/3}$ are optimal in $T$ and in the budget up to logarithmic factors. \end{corollary}
\begin{proof}
At fixed scales, Equation~\eqref{eq:gatechoice} makes $B$, $\Gamma$, $\Lambda$, and $\Lsm$ polylogarithmic in $mT$, so $K_T$, $K_T^+$, $\bar c$, $H_T$, and $H_T^+$ are polylogarithmic, while $T\zeta$, $T\delta$, and $\esm TV_\mu$ are $O(1)$. Thus $Te_f=4T\eopt+O(1)$ and $Te_f^+=10T\eopt+O(1)$.
\begin{itemize}
\item \emph{Counting.} On the Gaussian family $\App=0$, and Theorems~\ref{thm:main} and~\ref{thm:target} give $K_T(N+1)+Te_f$ against $c_*^\lambda R_0^2(N+1)$.
\item \emph{Relative-risk variation.} On the rare mixture with $K$ blocks, $K\asymp\sqrt{mT\Omega/R}$, Theorem~\ref{thm:main} gives $K_T+2\sqrt{10K_TT\Omega}+H_T\App+Te_f$, where $H_T\App=O(\log(eT)R^2K/m)$. The first three terms are a polylogarithmic multiple of $3KR/(64m)$ and of $\sqrt{T\Omega/m}$, up to a factor $\sqrt m$.
\item \emph{Target-law variation.} On the Gaussian family $\App=0$, and Theorem~\ref{thm:target} gives $2K_T^++2(2\bar c)^{1/3}(K_T^+)^{2/3}T^{1/3}V_\mu^{2/3}+Te_f^++O(1)$ against $(c_*^\lambda/34)T^{1/3}V_\mu^{2/3}m^{-2/3}$.
\end{itemize}
In all four cases the exponents of $T$ and of the budget coincide.
\end{proof}

\end{document}